\documentclass[11pt,a4paper]{article}
\usepackage{authblk}
\usepackage{graphicx}
\usepackage{amsmath}
\usepackage{amssymb}
\usepackage{algorithm}
\usepackage{algorithmic}
\usepackage{amsthm}
\usepackage{caption}
\usepackage{epstopdf}
\usepackage{url}
\usepackage{subfig}
\usepackage{natbib}
\usepackage{cases}
\usepackage{bm}
\usepackage{booktabs}
\usepackage{bbm}
\newcommand{\tabincell}[2]{\begin{tabular}{@{}#1@{}}#2\end{tabular}}
\newtheorem{theorem}{Theorem}
\newtheorem{lemma}{Lemma}

\theoremstyle{definition}

\title{ Model Embedding Model-Based Reinforcement Learning  }
\date{}
\author[1]{Xiaoyu Tan \footnote{Equal contribution with randomized order.}}
\author[1]{Chao Qu$^{*}$}
\author[1]{Junwu Xiong}
\author[1]{James Zhang}

\affil[1]{Ant Financial Services Group}

\begin{document}

\maketitle

\begin{abstract}
  Model-based reinforcement learning (MBRL) has shown its advantages in sample-efficiency over model-free reinforcement learning (MFRL). Despite the impressive results it achieves, it still faces a trade-off between the ease of data generation and model bias. In this paper, we propose a simple and elegant model-embedding model-based reinforcement learning (MEMB) algorithm in the framework of the probabilistic reinforcement learning. To balance the sample-efficiency and model bias, we exploit both real and imaginary data in the training. In particular, we embed the model in the policy update and learn $Q$ and $V$ functions from the real data set. We provide the theoretical analysis of MEMB with the Lipschitz continuity assumption on the model and policy. At last, we evaluate MEMB on several benchmarks and demonstrate our algorithm can achieve state-of-the-art performance.
\end{abstract}
\section{Introduction}

Reinforcement learning can be generally  classified into two categories:  model-free reinforcement learning (MFRL) and model-based reinforcement learning (MBRL). There is a surge of interest in MBRL recently due to its higher sample-efficiency comparing with MFRL \citep{kurutach2018model,heess2015learning,asadi2018lipschitz}. Despite its success,  MBRL still faces a challenging problem, i.e., the model-bias, where the imperfect dynamics model would degrade the performance of the algorithm \citep{kurutach2018model}. Unfortunately, such things always happen when the environment is sufficiently complex.  There are a few efforts to mitigate such issues by combining model-based and model-free approaches. \citet{heess2015learning} compute the value gradient along real system trajectories instead of planned ones to avoid the compounded error. \citet{kalweit2017uncertainty} mix the real data and imaginary data from the model and then train $Q$ function.  An ensemble of neural networks can be applied to model the environment dynamics, which effectively reduces the error of the model \citep{kurutach2018model,clavera2018model,chua2018deep}.

Indeed, how to exploit the real and imaginary data is a key question in model-based reinforcement learning. Recent model-based algorithms applying Dyna-style updates have demonstrated promising results  \citep{sutton1990integrated, kurutach2018model,luo2018algorithmic}. They collect real data using the current policy to train the dynamics model. Then the policy is improved using state-of-the-art \emph{model-free} reinforcement learning algorithms with imagined data generated by the learned model.  Our \emph{argument} is that why not \emph{directly embed} the model into the policy improvement?  To this end, we derive a reinforcement learning algorithm called model-embedding model-based reinforcement learning (MEMB) in the framework of the probabilistic reinforcement learning \citep{levine2018reinforcement}.

We provide the theoretical result on the error of the long term return in MEMB, which is caused by the model bias and policy distribution shift given the Lipschitz continuity condition of the model and policy. In addition, our analysis takes consideration of the length of the rollout step, which helps us to design the algorithm. In MEMB, the dynamics model and reward model are trained with the real data set collected from the environment.  Then we simply train $Q$ and $V$ function using the real data set with the update rule derived from the maximum entropy principle (several other ways to include the imaginary data can also be applied, see discussions in Section \ref{section:MEMB}). In the policy improvement step, the stochastic actor samples an action with the real state as the input, and then the state switches from $s$ to $s'$ according to the learned dynamics model.

We link the learned dynamics model, reward model, and  policy to compute an analytic policy gradient by the back-propagation. Comparing with the likelihood-ratio estimator usually used in the MFRL method, such value gradient method would reduce the variance of the policy gradient \citep{heess2015learning}. The other \emph{merit} of MEMB is its computational efficiency. Several state-of-the-art MBRL algorithms generate hundreds of thousands imaginary data from the model and a few real samples \citep{luo2018algorithmic,janner2019trust}. Then the \emph{huge} imaginary data set feeds into MFRL algorithms, which may be sample-efficient in terms of real samples  but not  computational-friendly. On the contrary,  our algorithm embeds the model in the policy update. Thus we can implement it efficiently by computing policy gradient several times in each iteration (see our algorithm \ref{alg:MEMB}) and do not need to do the calculation on the huge imaginary data set. 


 Notice SVG \citep{heess2015learning} also embeds the model to compute the policy gradient. However, there are \emph{several key differences} between MEMB and SVG. 
 \begin{itemize}
     \item To alleviate the issue of the compounded error, SVG proposes a  conservative algorithm where just real data is used to evaluate policy gradients and the imaginary data is wasted. However, our theorem shows that the  imaginary data from the short rollout from the learned model can be trusted. In our work, the policy is trained with the model and imaginary dataset $m$ times in each iteration of the algorithm. In the ablation study (appendix \ref{section:ablation_study}), we demonstrate such difference leads to a \emph{large} gap in the performance.
     \item We provide a theoretical guarantee of the algorithm, which is not included in SVG.  
     \item  We derive our algorithm in the framework of the probabilistic reinforcement learning. The entropy term would encourage the exploration, prevent the early convergence to the sub-optimal policies, and show state-of-the-art performance in MFRL \citep{haarnoja2018soft}. 
 \end{itemize}
  In addition, MEMB avoids the importance sampling in the off-policy setting by sampling the action from $\pi$ and transition from the learned model, which further reduces the variance of the gradient estimation.
  
\textbf{Contributions:} We derive an elegant, \emph{sample-efficient}, and \emph{computational-friendly} \footnote{We can finish one trial of the experiment around one or two hours on a laptop.} Dyna-style MBRL algorithm in the framework of the probabilistic reinforcement learning in a principled way. Different from the traditional MBRL algorithm, we directly \emph{embed} the model into the policy improvement, which could reduce the variance in the gradient estimation and avoid the computation on the huge imaginary data set. In addition, since the algorithm is \emph{off-policy}, it is sample-efficient. At last, we provide  theoretical results of our algorithm  on the long term return considering the model bias and policy distribution shift.  We test our algorithm on several benchmark tasks in Mujoco simulation environment \citep{todorov2012mujoco} and demonstrate that our algorithm can achieve state-of-the-art performance. We provide our code anonymously for the reproducibility \footnote{Code is submitted at https://github.com/MEMB-anonymous1/MEMB }. 

\textbf{Related work:} There are a plethora of works on MBRL. They can be classified into several categories depending on the way to utilize the model, to search the optimal policy or the function approximator of the dynamics model. We leave the comprehensive discussion on the related work in appendix \ref{appendix:related_work}.

\section{Preliminaries}
In this section, we first present some backgrounds on the Markov decision process. Then we introduce the knowledge on the probabilistic reinforcement learning with entropy regularization \citep{ziebart2008maximum,levine2018reinforcement} and stochastic value gradient \citep{heess2015learning} since parts of them are the building blocks of our algorithm.  
\subsection{MDP}

 Markov Decision Process (MDP) can be described by  a 5-tuple ($\mathcal{S}, \mathcal{A}, r, p, \gamma$): $\mathcal{S}$ is the  state space, $\mathcal{A}$ is the action space, $p$ is the transition probability, $r$ is the expected reward, and $\gamma\in [0,1)$ is the discount factor. That is for $s\in \mathcal{S}$ and $a\in \mathcal{A}$, $r(s,a)$ is the expected reward, $p(s'|s,a)$ is the probability to reach the state $s'$. A policy is used to select actions in the MDP. In general, the policy is stochastic and denoted by $\pi$, where $\pi(a_t|s_t)$ is the conditional probability density at $a_t$ associated with the policy. The state value evaluated on policy $\pi$ could be represented by $V^\pi(s)= \mathbb{E}_{\pi} [\sum_{t=0}^{\infty} \gamma^t r(s_t,a_t)| s_0=s]$ on immediate reward return $r$ with discount factor $\gamma\in (0,1)$ along the horizon $t$. When the entropy of the policy is incorporated in the probabilistic reinforcement learning \citep{ziebart2008maximum}, we could redefine the reward function $r(s,a)\leftarrow r(s,a)-\log\pi(a|s)$. When the model of the environment is learned from the data, we use $\hat{p}(s'|s,a)$ to denote the learned dynamic model, and $\hat{r}(s,a)$ as the learned reward model.
 
 We denote the \emph{true long term return} as $\eta(\pi):= \mathbb{E}_{\pi} \sum_{t=0}^{\infty} \gamma^t r(s_t,a_t)$, where the expectation corresponds to the policy, true transition and true reward.  In the model based reinforcement learning, we denote the \emph{model long term return} as $\hat{\eta}(\pi):=\hat{\mathbb{E}}_\pi \sum_{t=0}^{\infty}\gamma^t \hat{r}(s_t,a_t)$, where $\hat{\mathbb{E}}$ means the expectation over the policy, learned model $\hat{p}$ and $\hat{r}$. 
 
\subsection{Probabilistic Reinforcement Learning}\label{section:PRL}

\citet{levine2018reinforcement} formulate reinforcement learning as a probabilistic  inference problem. The trajectory $\tau$ up to time step $T$ is defined as $$\tau=((s_0,a_0),(s_1,a_1),...,(s_T,a_T)).$$
 The probability of the trajectory with the optimal policy is defined as  $$\rho = [p(s_0)\prod_{t=0}^{T}p(s_{t+1}|s_t,a_t)]\exp\big(\sum_{t=0}^{T} r(s_t,a_t)\big).$$The probability of the trajectory induced by the policy $\pi(a|s)$ is $$ \tilde{\rho}= p(s_0)\prod_{t=0}^{T}p(s_{t+1}|s_t,a_t)\pi(a_t|s_t).$$ The objective is to minimize the KL divergence $ KL(\tilde{\rho},\rho)$, which leads to the  entropy regularized reinforcement learning $\max_{\pi} \sum_{t=0}^T \mathbb{E}_{(s_t,a_t)\sim\rho_\pi}[r(s_t,a_t)+\alpha\mathcal{H}(\pi(\cdot|s_t))],$
where $\mathcal{H}(\pi(\cdot|s_t))$ is an entropy term scaled by  $\alpha$ \citep{ziebart2008maximum}. 
The optimal policy can be obtained by the following soft-Q update \citep{foxtaming}.
$$ Q(s_t,a_t) \longleftarrow r(s_t,a_t) + \gamma \mathbb{E}_{s_{t+1}\sim p}[V(s_{t+1})],V(s_t)\leftarrow \alpha \log(\int_{\mathcal{A}}  \exp(\frac{1}{\alpha} Q(s_t,a_t))da_t ).$$ 

Above iterations define the soft $Q$ operator, which is a contraction. The optimal policy $\pi^{*}(a|s)$ can be recovered by 
 $\pi^\star(a_t|s_t)= \frac{\exp(\frac{1}{\alpha}Q^*(s_t,a_t))}{\int_{\mathcal{A}}\exp(\frac{1}{\alpha}Q^*(s_t,a_t))da_t}$,
where $Q^*$ is the fixed point of the soft-Q update. We refer readers to the work  \citep{ziebart2008maximum,haarnoja2017reinforcement} for more discussions. In soft actor-critic \citep{haarnoja2018soft},  the optimal policy $\pi^*(a_t|s_t)$ is approximated by a neural network $\pi_{\theta}(a_t|s_t)$,  which is obtained by solving the following optimization problem $$\max_{\pi_\theta(a_t|s_t)} \mathbb{E}_{s_t\sim p(s_t)}\mathbb{E}_{a_t\sim \pi_\theta(a_t|s_t)} [Q(s_t, a_t)-\alpha\log\pi_\theta(a_t|s_t) )].$$ 

\subsection{Stochastic Value Gradient}
Stochastic value gradient method is a model-based algorithm which is designed to avoid the compounded model errors by only using the real-world observation and gradient information from the model \citep{heess2015learning}. The algorithm directly substitutes the dynamics model and reward model in the Bellman equation and calculates the gradient. To perform the backpropagation in the stochastic Bellman equation, the re-parameterization trick is applied to evaluate the gradient on real-world data. In SVG(1), the stochastic policy $\pi(a|s;\theta)$ with parameter $\theta$ could be optimized by the policy gradient in the following way

\begin{flalign}\label{equ:svg_gradient}
  \frac{\partial V(s)}{\partial \theta}\approx \mathbb{E}_{\eta, \zeta}[\frac{\partial \hat{r}(s,a)}{\partial a}
\frac{\partial \pi(a|s)}{\partial \theta}+\gamma(\frac{\partial V^\prime(s^\prime)}{\partial s^\prime}\frac{\partial f(s,a)}{\partial a}\frac{\partial \pi(a|s)}{\partial \theta})],
\end{flalign}
where  $\eta$ and $\zeta$ are the policy and model re-parameterization noise which could be directly sampled from a prior distribution or inferred from a generative model $g(\eta,\zeta|s,a,s^\prime)$. The $f(s,a)$ and $\hat{r}(s,a)$ are dynamics model and reward model respectively. In the off-policy update, SVG includes the important weight $ w=\frac{\pi(a_{k}|s_{k},\theta_t)}{ \pi(a_{k}|s_{k},\theta_k)}$, where $\theta_t$ represent the parameter of the current policy and $k$ is the index of the data from the replay buffer.

\textbf{Notions:} Given two metric space $(X, d_X)$ and $(Y, d_Y) $, we say a function $f$ is $L$ Lipschitz if $d_Y(f(x_1),f(x_2))\leq L d_X(x_1,x_2), \forall x_1,x_2 \in X$. Give a meritc space $(M,d)$ and the set $\mathbb{P}(M)$ of probability measures on $M$, the Wasserstein metric between two probability distributions $\mu_1$ and $\mu_2$ in $\mathbb{P}(M)$ is defined as $W(\mu_1,\mu_2):=\inf_{j\in \Sigma} \int\int p(x,y)d(x,y)dxdy,$
where $\Sigma$ denotes the collection of all joint distributions $p$ with marginal $\mu_1$ and $\mu_2$.

\section{MEMB}\label{section:MEMB}
In this section, we introduce our model-embedding model-based reinforcement learning algorithm (MEMB). Particularly, we optimize the following model long term return with the entropy regularization.
\begin{equation}\label{mbsac:objevtive}
\max_\pi\sum_{t=0}^T \hat{\mathbb{E}} [\hat{r}(s_t,a_t)+\mathcal{H}(\pi(a_t|s_t))],
\end{equation}
where we omit the regularizer parameter $\alpha$ of the entropy term in the following discussion to ease the exposition. Remind that optimizing the entropy regularzied reinforcement learning is equivalent to minimize the KL divergence between the distribution of $\rho$ and $\tilde{\rho}$ in Section \ref{section:PRL}. Now we replace the true model by the learned model $\hat{p}(s_{t_{t+1}}|s_{t},a_{t})$ and $\hat{r}(s,a)$.  Therefore we have $\tilde{\rho}_\pi = p(s_0)\prod_{t=0}^{T}\hat{p}(s_{t+1}|s_t,a_t)\pi(a_t|s_t)$ and $\rho_\pi = [p(s_0)\prod_{t=0}^{T}\hat{p}(s_{t+1}|s_t,a_t)]\sum_{t=0}^T \exp(\hat{r}(s_t,a_t))$. We then optimize the KL divergence $ KL(\tilde{\rho}, \rho) $, w.r.t to $\pi(a_t|s_t)$. Using the backward view as that in \citep{levine2018reinforcement}, we have the optimal policy. We defer the derivation to the appdendix \ref{appendix:derivation}.
      $$\pi^*(a_t|s_t)=\frac{\exp(Q(s_t,a_t))}{ \int_a \exp(Q(s_t,a_t) da_t},$$
\begin{equation}\label{equ:soft_bellman_mb}
\text{where}~ Q(s_t,a_t) = \hat{r}(s_t,a_t) + \gamma \mathbb{E}_{s_{t+1}\sim \hat{p}}[V(s_{t+1})],V(s_t)=\mathbb{E}_{\pi(a_t|s_t)}[Q(s_t,a_t)-\log\pi(a_t|s_t)].
\end{equation}

In the policy improvement step, the optimal policy $\pi^* $ can be approximated by a parametric function $\pi_{\theta} (a_t|s_t)$. In the \emph{MFRL}, this can be obtained by solving $$\max_{\pi_\theta(a_t|s_t)} \mathbb{E}_{s_t\sim p(s_t)}\mathbb{E}_{a_t\sim \pi_\theta(a_t|s_t)} [Q(s_t, a_t)-\log\pi_\theta(a_t|s_t) )]$$ \citep{levine2018reinforcement,haarnoja2018soft}. A straightforward way is to optimize $Q$, $V$ and $\pi$ using the imaginary data from the rollout, which \emph{reduces} to \citet{luo2018algorithmic,janner2019trust} and many others. However such way used in the MFRL \emph{cannot} leverage the model information. We leave the our derivation and discussion on the policy improvement in  Section \ref{section:policy_learning}.

\subsection{Model Learning}\label{section:model_learning}
The transition dynamics and rewards could be modeled by non-linear function approximations as two independent regression tasks which have the same input but different output. Particularly, we train two independent deep neural networks with parameter $\omega$ and $\varphi$ to represent the dynamics model $\hat{p}$ and reward model $\hat{r}$ respectively. In our analysis, we assume $\hat{p}$ is not far from $p$, i.e., $W(p(s'|s,a), \hat{p}(s'|s,a))\leq \epsilon_m$, which can be estimated in practice by cross validation.

To better represent the stochastic nature of the dynamic transitions and rewards, we implement re-parameterization trick on both $\hat{p}$ and $\hat{r}$ with input noises $\zeta_\omega$ and $\zeta_\varphi$ sampled from Gaussian distribution $\mathcal{N}(0,1)$. Thus we can denote dynamic model $s'=f(s,a,\zeta_{w})$ and reward as $\hat{r}(s,a,\zeta_{\varphi})$. In practice, we use neural networks to  generate mean: $\mu_\omega$, $\mu_\varphi$, and variance: $\sigma_\omega$, $\sigma_\varphi$ separately for the transition model and reward model, respectively. Then, we compute the result by $\mu_\omega+\sigma_\omega\zeta_\omega$ and $\mu_\varphi+\sigma_\varphi\zeta_\varphi$, respectively.

We optimize above two models by sampling the data $(s,a,s',r)$ from the (real data) replay buffer $\mathcal{D}$ and minimizing the mean square error: 
\begin{flalign}\label{equ:model_loss}
J(\omega) = \frac{1}{2}\mathbb{E}_{ \mathcal{D},\zeta_\omega}[(f(s,a,\zeta_{\omega})-s^\prime)^2],J(\varphi) = \frac{1}{2}\mathbb{E}_{ \mathcal{D},\zeta_\varphi}[(\hat{r}(s,a,\zeta_\varphi)-r)^2].
\end{flalign} 

\subsection{Value function learning}\label{section:value_function_learning}

Although in equation \eqref{equ:soft_bellman_mb}, the $Q$ function is computed w.r.t. the learned model, this way would cause the model bias in practice. In addition, in the policy update, this bias would result in an additional error of the policy gradient (since roughly speaking, the gradient of the policy is weighted by the Q function). To avoid this model error, we update $V$, $Q$ using equation \eqref{equ:soft_bellman_mb} with the real transition $(s, a, r, s') $ from the real data replay buffer. Particularly, we minimize the following error w.r.t to $V$ and $Q$.
\begin{flalign}\label{equ:mbsac_v}
    J(\psi)= \mathbb{E}_{s_t \sim \mathcal{D}}[\frac{1}{2}(V_\psi(s_t)-\mathbb{E}_{a_t\sim\pi}[Q_\phi(s_t,a_t)-\log\pi(a_t|s_t)])^2].
\end{flalign}
\begin{equation}\label{equ:mbsac_q_1}
    J(\phi)=\mathbb{E}_{(s_t,a_t)\sim\mathcal{D}}[\frac{1}{2}\big(Q_\phi(s_t,a_t)-r_t-\gamma V_{\psi}(s_{t+1}) \big)^2]. 
\end{equation}
 A straightforward way to incorporate the imaginary data is the value expansion  \citep{feinberg2018model}. However in our ablation study, we find that the training with real data gives the best result. Thus we just briefly introduce the value expansion here. If $Q$ function and $V$ function are parameterized by $\phi$ and $\psi$ respectively, they could be updated by minimizing the new objective function with the value expansion on imaginary rollout: $  J(\phi)=\frac{1}{H}\sum_{t=0}^{H-1}\big(Q_\phi(\hat{s}_t,\hat{a}_t)-(\sum_{k=t}^{H-1}\gamma^{k-t}\hat{r}_k+\gamma^{H-t}V_\psi(\hat{s}_H))\big)^2,$ where only the initial tuple $\tau_0=(\hat{s}_0,\hat{a}_0,\hat{r}_0,\hat{s}_1)$ is sampled from replay buffer $\mathcal{D}$ with real-world data, and later transitions are sampled from the imaginary rollout from the model.  $H$ here is the time step of value expansion using imaginary data. $\tau$ is the training tuple and $\tau_0$ is the initial training tuple. Note that when $H=1$, it reduces to the case where just real data is used. 
 
 \subsection{Policy Learning}\label{section:policy_learning}
Then we consider the policy improvement step, i.e., to calculate the optimal policy at each time step. One \emph{straightforward} way is to optimize the following problem $$\max_{\pi(a_t|s_t)} \mathbb{E}_{s_t\sim p(s_t)}\mathbb{E}_{a_t\sim \pi(a_t|s_t)} [Q(s_t, a_t)-\log\pi(a_t|s_t) )]$$ as that in MFRL but with the imaginary data set. This way reduces to the work \citep{luo2018algorithmic,chua2018deep,janner2019trust}. However, such way cannot leverage the learned dynamics model and reward model. To incorporate the model information, notice that $V(s_t)=\mathbb{E}_{a\sim \pi(a_t|s_t)}[Q(s_t,a_t)-\log \pi(a_t|s_t)]$, thus the policy improvement step is equal to 
\begin{equation}\label{equ:policy_update}
\max_{\pi(a_t|s_t)}\mathbb{E}_{s_t\sim p(s_t)}  V(s_t).
\end{equation}
In the following, we connect the dynamics model, reward model, and value function together by the soft Bellman equation. Recall we have re-parameterized the dynamics model $s'=f(s,a,\zeta_\omega)$ in  Section \ref{section:model_learning}. 
Now we re-parameterize the policy as  $a=\kappa(s,\eta;\theta)$  with noise variables $\eta\sim \rho (\eta) $. Now we can write the soft Bellman equation in the following way.
\begin{equation}\label{equ:bellman_repar}
    V(s)=\mathbb{E}_{\eta}[\mathbb{E}_{\zeta_\varphi}\hat{r}(s, \kappa(s,\eta;\theta),\zeta_\varphi)-\log\pi(a|s)
+\gamma\mathbb{E}_{\zeta_w}V'(f(s,\kappa(s,\eta;\theta),\zeta_\omega)) ]
\end{equation}
To optimize \eqref{equ:policy_update} and leverage  gradient information of the model, we sample $s$ from the real data replay buffer $D$ and take the gradient of $V(s)$ w.r.t. $\theta$ 
\begin{equation} \label{equ:mbsac_re_policy}
    \mathbb{E}_{s\sim D}\frac{\partial V(s)}{\partial \theta}=\mathbb{E}_{s\sim D, \eta, \zeta_w,\zeta_\varphi}[\frac{\partial \hat{r}}{\partial a}\frac{\partial \kappa}{\partial \theta}-\frac{1}{\pi}\frac{\partial \pi}{\partial \theta}+\gamma(\frac{\partial V^\prime(s^\prime)}{\partial s^\prime}\frac{\partial f}{\partial a}\frac{\partial \kappa}{\partial \theta})].
\end{equation}
For the Gaussian noise case, we can further simplify $\frac{\partial \pi}{\partial \theta}$ by plugging in the density function of the normal distribution. Clearly, we can unroll the Bellman equation with $k$ steps and obtain similar result but here we focus on $k=1$. The equation \eqref{equ:mbsac_re_policy} demonstrates an interesting connection between our algorithm and SVG. Notice that the transition from $(s,a)$ to $s'$ is sampled from the learned dynamics model $f$, while the SVG(1) just utilizes the real data. Thus in the algorithm we can update policy several times in each iteration to fully utilized the model rather than just use the real transition once.  Compared with the policy gradient step taken by SVG(1) algorithm \citep{heess2015learning}, equation \eqref{equ:mbsac_re_policy} includes one extra term $-(1/\pi)(\partial\pi/\partial\theta)$ to maximize the entropy of policy. We also drop the importance sampling weights by sampling from the current policy.

\subsection{MEMB algorithm}
We summarize our MEMB in Algorithm \ref{alg:MEMB}. At the beginning of each step, we train dynamics model $f$ and reward model $\hat{r}$ by minimizing the $\text{L}_2$ loss shown in \eqref{equ:model_loss}. Then the agent interacts with the environment and stores the data in the real data replay buffer $D$. Actor samples $s_k$ from $D$ and collects $s_{k+1}$ according to the dynamics model $f(s_k,a_k,\zeta_\omega)$. Such imaginary transition is stored in $D_{img}$. Then we train $Q$, $V$, and $\pi$ according to the update rule in Section \ref{section:MEMB}.  Similar to other value-based RL algorithms, our algorithm also utilizes two $Q$ functions to further reduce the overestimation error by training them simultaneously with the same data but only selecting the minimum target in value updates \citep{fujimoto2018addressing}. We use the target function for $V$ like that in deep Q-learning algorithm \citep{mnih2015human}, and update it with an exponential moving average. We train policy using the gradient in \eqref{equ:mbsac_re_policy}. Remark that our $s'$ is sampled from the dynamic model $f(s,a,\zeta_\omega)$, while in SVG, it uses the true transition.

\begin{algorithm}[h]
	\caption{MEMB}\label{alg:MEMB}
	\begin{algorithmic}
		\STATE {\textbf{Inputs}: Replay buffer $\mathcal{D}$,  imaginary replay buffer $\mathcal{D}_{img}$, policy $\pi_\theta$, value function $V_\psi$, target value function $V_{\bar{\psi}}$.  Two $Q$ functions with parameters $\phi_0$ and $\phi_1$, dynamic model $f$ with parameter $\omega$, and reward model $\hat{r}$ with parameter $\varphi$}\\
		\FOR{each iteration}
		\STATE{}
		\textbf{1. Train the dynamics model and reward model}\\
		\STATE{Calculate the gradients $\nabla_{\omega}J(\omega)$, $\nabla_{\varphi}J(\varphi)$ using \eqref{equ:model_loss} with $\mathcal{D}$, update $\omega$ and $\varphi$ }
		\STATE{}
		\textbf{2. Interact with environment}
		\STATE{Sample $a_t \sim \pi(a_t|s_t)$, get reward $r_t$, and observe the next state $s_{t+1}$}
		\STATE{Append the tuple $(s_t,a_t,r_t,s_{t+1})$ into $\mathcal{D}$}\\
		\textbf{3. Update the actor, critics $m$ times }
		\STATE{Empty $\mathcal{D}_{img}$}
		\STATE{Sample $(s_0,a_0,r_0,s_1) \sim \mathcal{D}$}
		\FOR{each imaginary rollout step $k$}
		\STATE{Sample $a_k \sim \pi(a_k|s_k)$, get reward $r_k=\hat{r}(s_k,a_k,\zeta_\varphi)$, and sample $s_{k+1} \sim f(s_k,a_k,\zeta_\omega)$}
		\STATE{Append the tuple $(s_k,a_k,r_k,s_{k+1})$ into $\mathcal{D}_{img}$}	
		\ENDFOR
		\STATE{Calculate the gradient $\nabla_{\phi}J(\phi)$ using \eqref{equ:mbsac_q_1} with $\bar{\psi}$ and $\mathcal{D}_{img}$.}
		\STATE{Calculate the gradient $\nabla_{\psi}J(\psi)$ using \eqref{equ:mbsac_v} with $\mathcal{D}$}
		\STATE{Calculate the gradient $\nabla_{\theta}V(s)$ using \eqref{equ:mbsac_re_policy} with $\mathcal{D}_{img}$.}
		\STATE{Update $\phi$, $\psi$, and $\theta$, update $\bar{\psi}$ with Polyak averaging}
		\ENDFOR
	\end{algorithmic}
\end{algorithm}
	\section{Theoretical Analysis}
	In this section, we provide a theoretical analysis for our algorithm. Notice our algorithm basically samples a state from the real data replay buffer and then unrolls the trajectory with several steps. Then using this imaginary data from rollout and the value function trained from real data, we update the policy with our policy learning formulation in \ref{section:policy_learning}. In the following, we first give a general result on how accurate the model long term return is regardless of how many rollout step is used in our algorithm. Later, we provide a more subtle analysis considering the rollout procedure.

	We first investigate the difference between the true long term return $\eta$ and the model long term return $\hat{\eta}$, which is induced by the model bias and the distribution shift due to the updated policy encountering states not seen during model training. Particularly, we denote that the model bias as  $W(p(s'|s,a), \hat{p}(s'|s,a))$, i.e., the Wasserstein distance between the true model and the learned model. We denote the distribution shift as $W(\pi(a|s),\pi_D(a|s))$, where $\pi_D$ is the data-collecting policy and $\pi$ is intermediate policy during the update of the algorithm.  For instance, in our algorithm \ref{alg:MEMB}, $\pi_D$ corresponds to the replay buffer of the true data while $\pi$ is the policy in the imaginary rollout. We assume $W(p(s'|s,a), \hat{p}(s'|s,a))\leq \epsilon_m, \forall s,a$ and $W(\pi(a|s),\pi_D(a|s))\leq \epsilon_\pi, \forall s$.  Comparing with the total variation used in \citep{janner2019trust}, the Wasserstein distance has better representation in the sense of how close $\hat{p}$ approximate $p$ \citep{asadi2018lipschitz}. For instance, if $p$ and $\hat{p}$ has disjoint supports, the total variation is always 1 regardless of how far the supports are from each other. Such case could always happen in the high-dimensional setting \citep{gulrajani2017improved}. 
	
	To bound the error of the long term return, we follow the Lipschitz assumption on the model and policy in \citep{asadi2018lipschitz}. Particularly, a transition model $\hat{p}$ belongs to Lipschitz class model if it is represented by
	$ \hat{p}(s'|s,a)=\sum_{f_m} \mathbbm{1}(f_m(s,a)=s')g_m(f_m) ,$ which says the transition probabilities can be decomposed in to a distribution $g$  over a set of deterministic function $f_m$. The model is called $K_m$ Lipschitz, if $f_m$ is a $K_m$ Lipschitz function. It is easy to understand this in the context of our work when we re-parametrize the model function. For instance, if the transition is deterministic, i.e., $s'=f_m(s,a)$, then $K_m$ is the Lipschitz constant of $f_m$.  Similarly the policy $\pi$ associated with Lipschitz class is given by $\pi(a|s)=\sum_{f_\pi}\mathbbm{1}(f_\pi(s)=a)g_{\pi}(f_\pi) $ and $f_\pi(s)$ is $K_\pi$ Lipschitz. If we use neural network to approximate model and policy, Lipschitz continuity means the gradient w.r.t the input is bounded by a constant.  Here for simplicity, we assume $g_m$ and $g_{\pi}$ are independent with state and action. The similar bound including such dependence can be proved but with more involved assumption and notations.
	
	Notice the analysis in  \citep{asadi2018lipschitz} just considers the error caused by the model bias and neglect the the effect of distribution shift of the policy. Therefore they assume $\pi_D=\pi$ and thus do not need the Lipschitz assumption on $\pi$. In addition, we give a bound considering the $k$ step rollout in Theorem \ref{theorem:branched_result}, which is not covered by \citep{asadi2018lipschitz}. In the following, we first give a general result to bound true long term return $\eta$ and model long term return $\hat{\eta}$, where we assume the reward model is known and mainly focus on the effect of the model bias and distribution shift. 
	
	\begin{theorem}\label{theorem:general_result}
		Let the true transition model $p$ and the learned transition model $\hat{p}$ both be $K_m$ Lipschitz. We also assume policy $\pi(a|s)$ and reward function $r(s,a)$ are $K_\pi$ and $ K_r$ Lipschitz respectively. Suppose $W(\hat{p}(s'|s,a), p(s'|s,a))\leq \epsilon_m, \forall (s,a)$, $W(\pi(a|s),\pi_D(a|s))\leq \epsilon_\pi, \forall s$. Let $\bar{K}:=K_{\pi}K_{m}$ and assume $\bar{K}\in [0,1) $, then the difference between true return and model return is bounded as $|\eta(\pi)-\hat{\eta}(\pi)|\leq 2\big[\frac{\gamma K_r \bar{K}}{(1-\gamma)(1-\gamma\bar{K})}+\frac{K_r}{1-\gamma} \big]\epsilon_{\pi}+\frac{\gamma K_r\bar{K}}{(1-\gamma)(1-\gamma \bar{K})}\epsilon_m. $  
	\end{theorem}

	Notice above theorem is a generic result on the model based reinforcement learning.   Such analysis is based on running \emph{full} rollout of the learned model, which results in the compounded error. However, notice in our algorithm, we actually start a rollout from a state with the distribution induced by the previous policy $\pi_D$ and then run imaginary rollout for k steps using the current policy $\pi$ on learned transition model $\hat{p}$. Follow the notion in \citep{janner2019trust}, we call it k-step branched rollout.  We use $\eta^{branch}$ to describe the long term return of this branched rollout and have a fine-grained analysis in the following.
	
	\begin{theorem}\label{theorem:branched_result}
		Let the true transition model $p$ and learned transition model $\hat{p}$ both be $K_m$ Lipschitz. We also assume policy $\pi(a|s)$ and reward function $r(s,a)$ are $K_\pi$ and $ K_r$ Lipschitz respectively. Suppose $W(\hat{p}(s'|s,a), p(s'|s,a))\leq \epsilon_m, \forall (s,a)$, $W(\pi(a|s),\pi_D(a|s))\leq \epsilon_\pi, \forall s$. If $\bar{K}:=K_{\pi}K_{m}\in [0,1) $, then the difference between true return and branched return is bounded as $|\eta(\pi)-\eta^{branch}(\pi)|\leq K_r[\frac{\gamma^{k+1} \bar{K}}{(1-\gamma)(1-\bar{K}\gamma)}+\frac{\gamma^k}{1-\gamma}]\epsilon_\pi+K_r[\frac{K_\pi(1-\gamma^k)}{(1-\bar{K})(1-\gamma)}-\frac{K_\pi(1-(\gamma \bar{K})^k)}{(1-\gamma\bar{K})(1-\bar{K})}+\gamma^k \frac{K_\pi(1-\bar{K}^k)}{ (1-\bar{K})(1-\gamma)}]\epsilon_m.$  
	\end{theorem}
	Clearly, on the right hand side of the bound, some terms increase with the rollout step $k$ while the others decreases. Thus there exist a best $k^*$ which depends on the discount factor $\gamma$, $\epsilon_m$ and $\epsilon_\pi$.  In general, the imaginary data from short rollout is still trustful. Recall we can apply the rollout of Bellman equation in two different ways in our algorithm: (1) Policy update (equation \eqref{equ:bellman_repar}). (2) value function learning in section \ref{section:value_function_learning}. So we do ablation study on these two ways and find that the imaginary data in value function learning would degrade the performance while improves the learning a lot in policy update, which also explains why MEMB is much better than SVG.  Another interesting result is on $\epsilon_m$. In our algorithm \ref{alg:MEMB}, in each iteration, we can update policy $m$ times using the imaginary data from the model. If $m$ is too large, it will cause a large distribution shift $\epsilon_\pi$ and degrade the performance. As such, typically we choose $m$ as $3$ to $5$.
	
	\begin{figure*}[t]
\vspace{-5mm}
   \centering
    \subfloat[Inverted Pendulum]{\includegraphics[width=.30\textwidth]{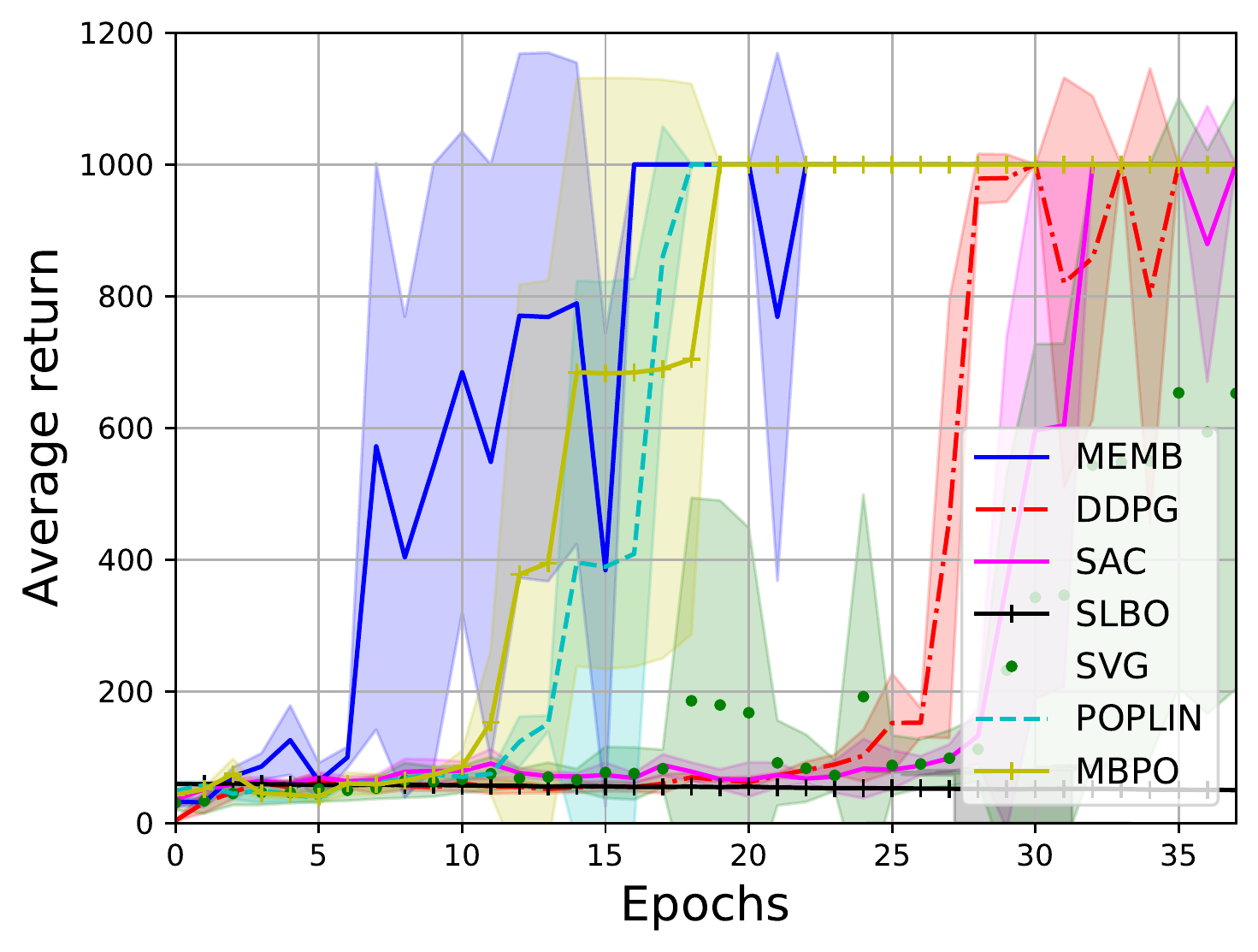}}\quad
    \subfloat[HalfCheetah]{\includegraphics[width=.3\textwidth]{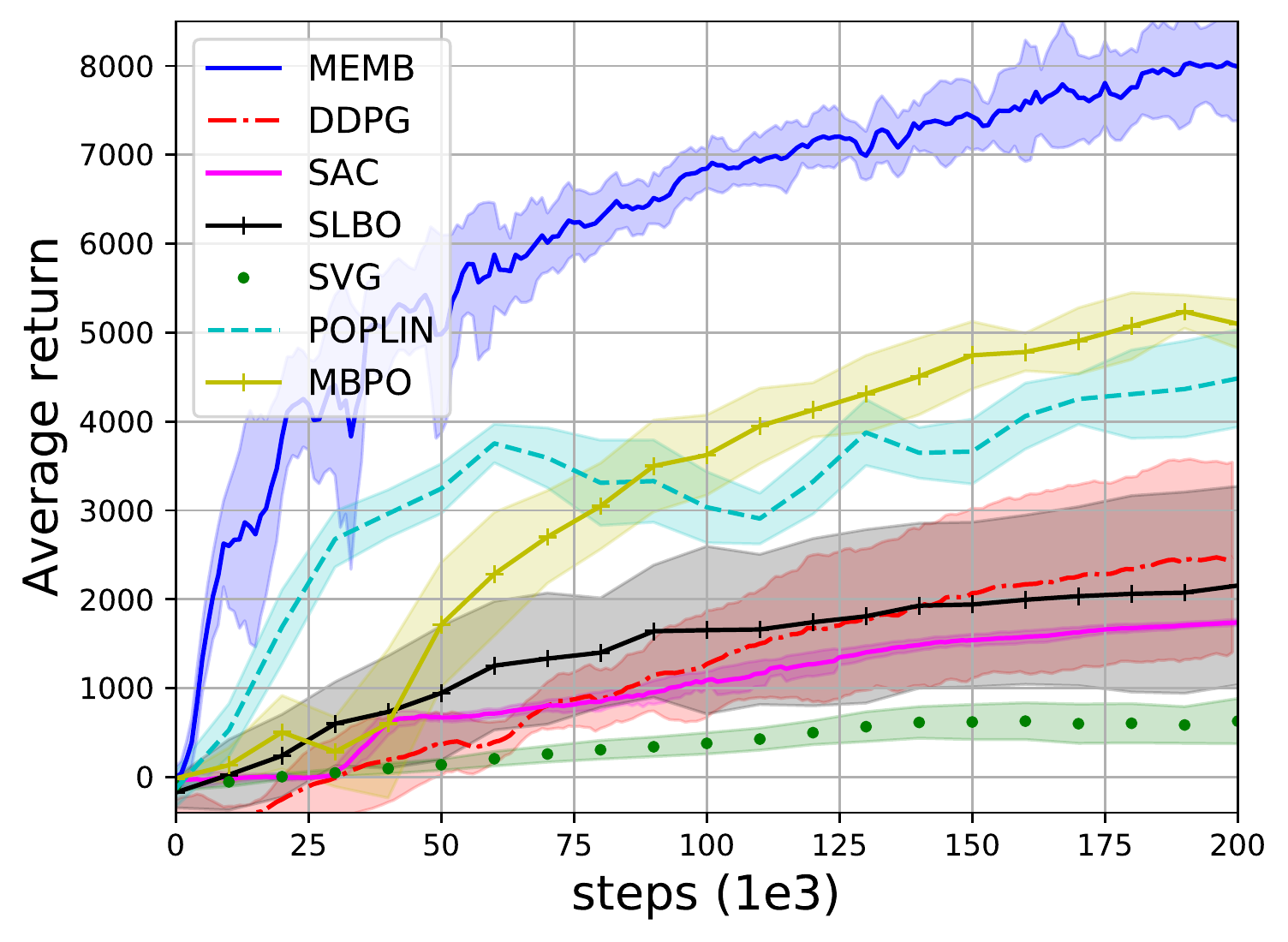}}\quad
   \subfloat[Reacher]{\includegraphics[width=.305\textwidth]{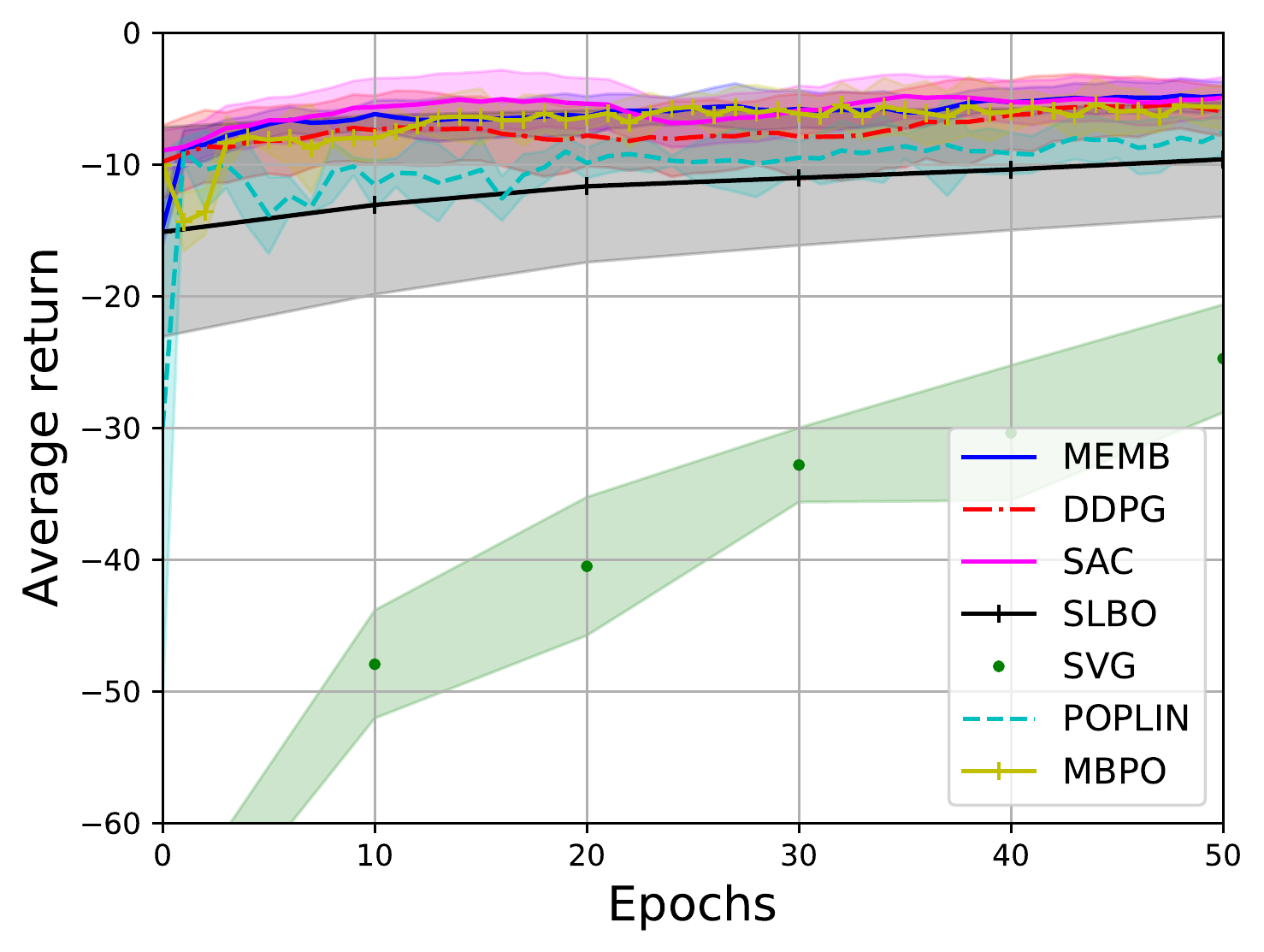}}\quad
    \vspace{-3mm}
    \subfloat[Hopper]{\includegraphics[width=.3\textwidth]{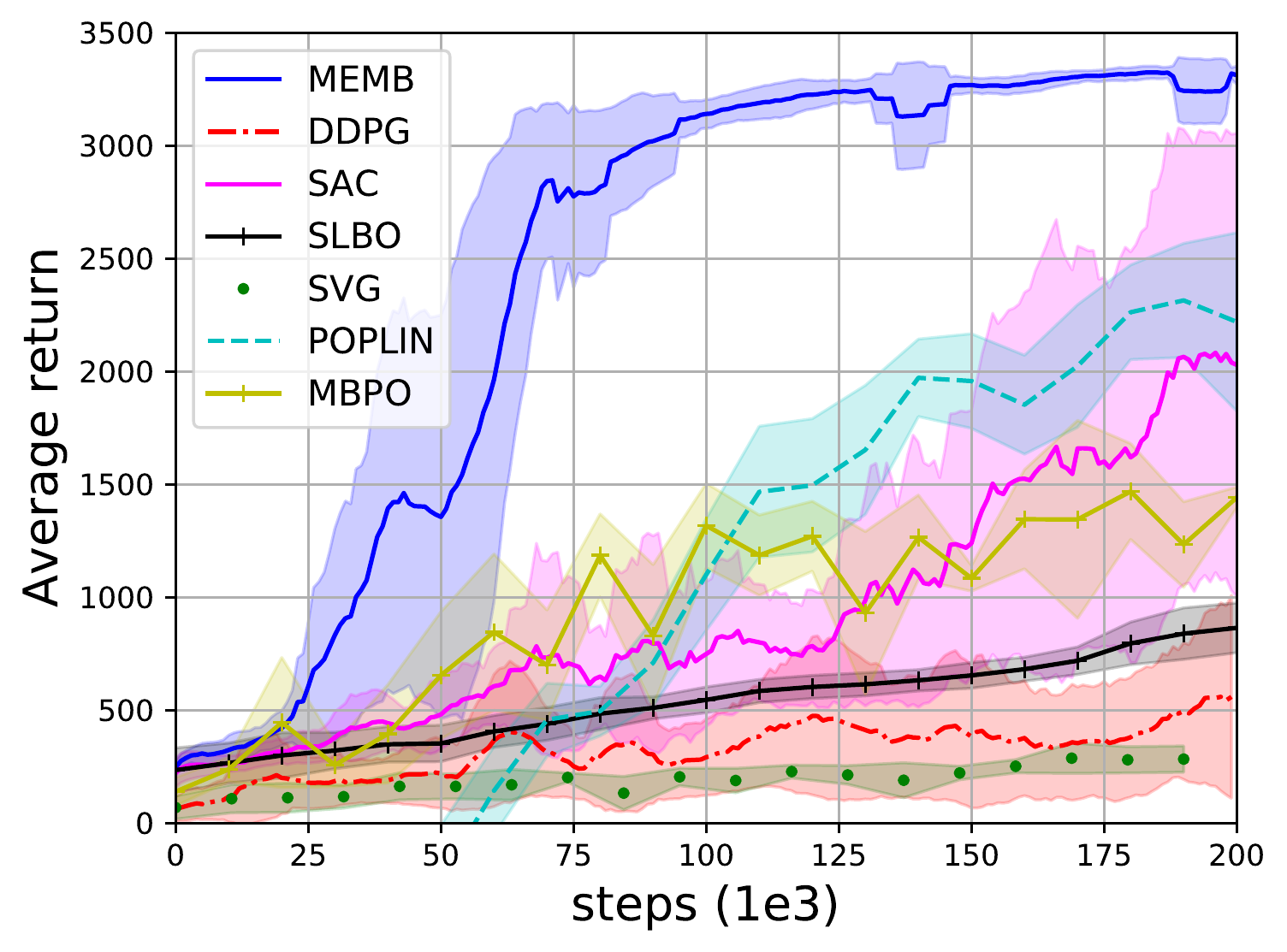}}\quad
    \subfloat[Swimmer]{\includegraphics[width=.3\textwidth]{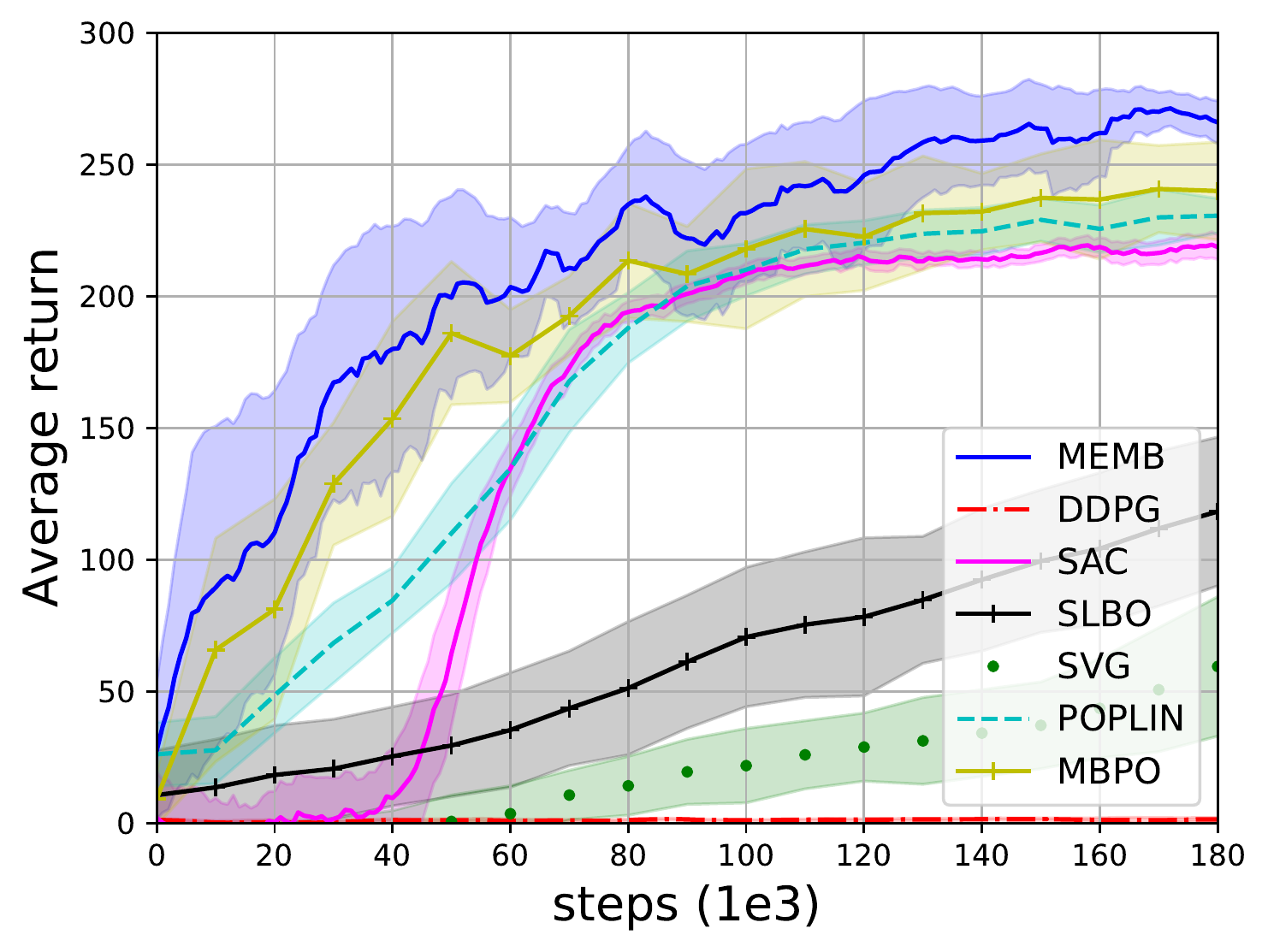}}\quad
   \subfloat[Walker2D]{\includegraphics[width=.3\textwidth]{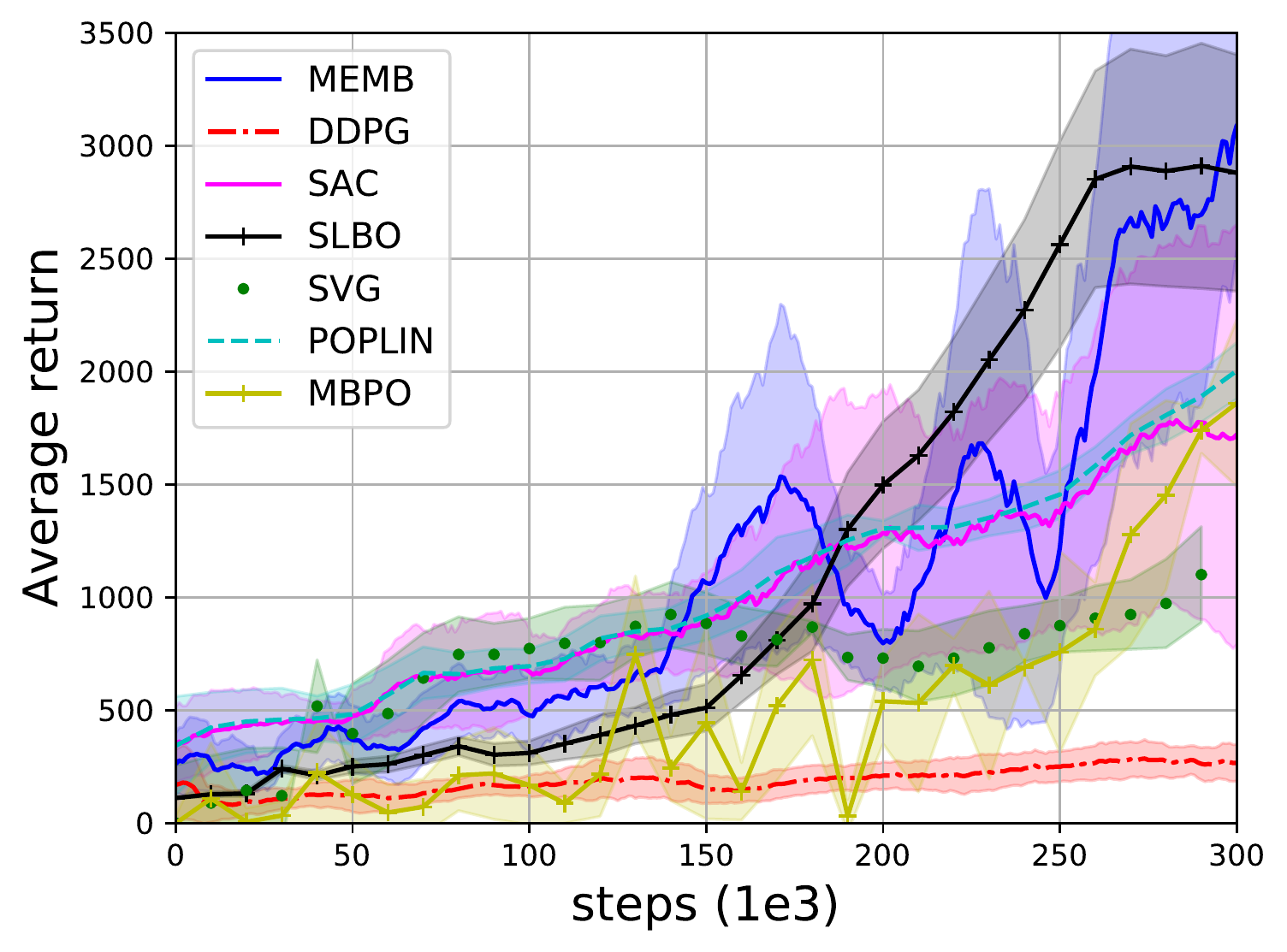}}\quad
   \caption{Performance of  MEMB and other baselines in benchmark tasks. The x-axis is the training step (epoch or step). Each experiment is tested on five trials using five different random seeds and initialized parameters.  For a simple task, i.e., InvertedPendulum, we limit the training steps at 40 epochs. For the other three complex tasks, the total training steps are 200K or 300K. The solid line is the mean of the average return. The shaded region represents the standard deviation. On  HalfCheetah, Hopper, and Swimmer, MEMB outperforms the other baselines significantly. In the task Walker2d, SLBO is slightly better than MEMB. They both surpass other algorithms. On Reacher, MEMB and SAC perform best. 
}
   \label{fig:experiment}
  	\vspace{-2mm}
\end{figure*}

\section{Experimental results}

 In this section, we would like to answer two questions: (1) How does MEMB perform on some benchmark reinforcement learning tasks comparing with other state-of-the-art model-based and model-free reinforcement learning algorithms? (2) Whether we should use the imaginary data generated by the model embedding in the policy learning in Section \ref{section:policy_learning}. How many imaginary data we should use in the value function update in Section \ref{section:value_function_learning}? We leave the answer of second question in the ablation study in appendix \ref{section:ablation_study}.
 
\textbf{Environment:}
To answer these two questions, we experiment in the Mujoco simulation environment \citep{todorov2012mujoco}: InvertedPendulum-v2, HalfCheetah-v2, Reacher-v2, Hopper-v2, Swimmer-v2, and Walker2d-v2. Each experiment is tested on five trials using five different random seeds and initialized parameters. The details of the tasks and experiment implementations can be found in appendix \ref{appendix:experiment}.

\textbf{Comparison to state-of-the-art}: We compare our algorithm with state-of-the-art model-free and model-based reinforcement learning algorithms in terms of sample complexity and performance. DDPG \citep{lillicrap2015continuous} and SAC \citep{haarnoja2018soft} are two model-free reinforcement learning algorithms on continuous action tasks. SAC has shown its reliable performance and robustness on several benchmark tasks. Our algorithm also builds on the maximum entropy reinforcement learning framework and benefits from incorporating the model in the policy update. Four model-based reinforcement learning baselines are SVG \citep{heess2015learning},SLBO \citep{luo2018algorithmic}, MBPO \citep{janner2019trust} and POPLIN \citep{wang2019exploring}. Notice in SVG, the algorithm just computes the gradient in the real trajectory, while our MEMB updates policy using the imaginary data $m$ times generated from the model. At the same time, we avoid the importance sampling by using the data from the learned model.  SLBO is a model-base algorithm with performance guarantees that applies TRPO \citep{schulman2015trust} on the data set generated from the rollout of the model. MBPO has the similar spirit but with SAC as the learning algorithm on the imaginary data.
 
For fairness, we compare the baseline without the ensemble learning techniques \citep{chua2018deep}. These techniques are known to reduce the model bias. We do not use distributed RL either to accelerate the training. We believe that the above-mentioned skills are orthogonal to our work and could be integrated into the future work to further improve the performance.  We just compare this pure version of MEMB with other baselines.  We also notice that some recent works in MBRL modify the benchmarks to shorten the task horizons and simplify the model problem while some work assume the true terminal condition is known to the algorithm \citep{wang2019benchmarking}. On the contrary, we test our algorithm in the full-length tasks and do not have assumptions on the terminal condition.

 We present experimental results in Figure \ref{fig:experiment}. In a simple task, InvertedPendulum, MEMB achieves the asymptotic result just using 16 epochs. In HalfCheetah,  MEMB's performance is at around 8000 at 200k steps, while all the other baselines' performance is below 5300. In Reacher, MEMB and SAC have similar performance. Both of them are better than other algorithms.  In Hopper, the final performance of MEMB is around 3300. The runner-up is POPLIN whose final performance is around 2300. In Swimmer, the performance of MEMB is the best. In Walker2d, SLBO is slighter better than MEMB. Both of them achieve the average return of 2900 at 300k timesteps. 

\bibliography{MBRL}
\bibliographystyle{icml2020}

\clearpage
\onecolumn
\appendix

\section{Related work}\label{appendix:related_work}
There are a plethora of works on MBRL. They can be classified into several categories depending on the way to utilize the model, to search the optimal policy or the function approximator of the dynamics model. Iterative Linear Quadratic-Gaussian (iLQG) \citep{tassa2012synthesis} assumes that the true dynamics are known to the agent. It approximates the dynamics with linear functions and the reward function with quadratic functions. Hence the problem can be transferred into the classic LQR problem. In Guided Policy Search \citep{levine2013guided,levine2014learning,finn2016guided}, the system dynamics are modeled with the time-varying Gaussian-linear model.  It approximated the policy with a neural network $\pi$ by minimizing the KL divergence between iLQG and $\pi$. A regularization term is augmented into the reward function to avoid the over-confidence on the policy optimization. Nonlinear function approximators can be leveraged to model more complicated dynamics. \citet{deisenroth2011pilco} use Gaussian processes to model the dynamics of the environment. The policy gradient can be computed analytically along the training trajectory. However, it may suffer from the curse of dimensionality which hinders its applicability in the real problem.  Recently, more and more works incorporate the deep neural network into MBRL. \citet{heess2015learning} model the dynamics and reward with neural networks, and compute the gradient with the true data. \citet{richards2005robust,nagabandi2018neural} optimize the action sequence to maximize the expected planning reward along with the learned dynamics model and then the policy is fine-tuned with TRPO. \citet{luo2018algorithmic,chua2018deep,kurutach2018model,janner2019trust} use the current policy to gather the data from the interaction with the environment and then learn the dynamics model. In the next step, the policy is improved (trained by the model-free reinforcement learning algorithm) with a large amount of imaginary data generated by the learned model. MEMB may reduce to their work by updating the policy with a model-free algorithm. \citet{janner2019trust} provide an error bound on the long term return of the k-step rollout given that the total variation of model bias and policy distribution are bounded by $\epsilon$. However, as we discussed, total variation may not be a good measure to describe the difference of learned model and true model especially when the support of the distributions is disjoint.  Ensemble learning can also be applied to further reduce the model error. \citet{asadi2018lipschitz} leverage the Lipschitz model to analyze the model bias in the long term return. Our work consider the both model bias and the distribution shift on the policy. In addition, we analyze the k-step rollout while \citet{asadi2018lipschitz} just gives the result on full rollout.

\section{Ablation Study}\label{section:ablation_study}
\subsection{How to utilize the imaginary data}
In this section, we make the ablation study to understand how much imaginary data we should include in the algorithm. Remind that in our algorithm, the model is embedded in the Soft Bellman equation in the policy update step, which means we fully trust the model to compute the policy gradient. This way to utilize the imaginary data would improve the sample-efficiency. To confirm our claim, we compare the MEMB with SVG on the task HalfCheetah. Notice when we just utilize the real trajectory in the policy update, MEMB reduces to the SVG\footnote{It still has some difference such as the update rule of Q and V function}. To remove the effect of the entropy regularization, we vary $\alpha$ (the regularizer parameter of entropy) in MEMB.  When $\alpha=0$, it reduces to the non-regularized formulation. In that case, we add the noise in policy for the exploration. We report the result in panel (a) of Fig \ref{fig:ablation}.  It is clear that using imaginary data in policy update improves the learning with a wide margin.

 In Section \ref{section:value_function_learning}, we train $Q$ and $V$ with the true data set. In the experiment, we also try the value expansion introduced in \citep{feinberg2018model}. We test the algorithm with value expansion, particularly with horizon $H=2$ and $H=5$. Our conclusion is that including the imaginary data to train the value function in our algorithm would hurt the performance, especially in the complex tasks. We demonstrate the performance of MEMB with value expansion in panel (b) and (c) of Figure \ref{fig:ablation}. We first test the algorithm on a simple task Pendulum from OpenAI gym \citep{brockman2016openai} and show the result in Panel (b) of Figure \ref{fig:ablation}. MEMB with $H=1$ converges to the optimal policy within several epochs. When we increase the value of $H$, the performance decreases. Then we evaluate the performance of value expansion in a complex task HalfCheetah from the Mujoco environment \citep{todorov2012mujoco} in panel (c) of Figure \ref{fig:ablation}.  In this task, value expansion with $H=2$ and $H=5$ does not work at all. The reason would be that the dynamics model of HalfCheetah introduces more significant model bias comparing to the simple task Pendulum. Thus training both policy and value function in the imaginary data set may cause a large error in policy gradient.

 \begin{figure*}[t]
\centering
\subfloat[HalfCheetah (policy)]{\includegraphics[width=.315\textwidth]{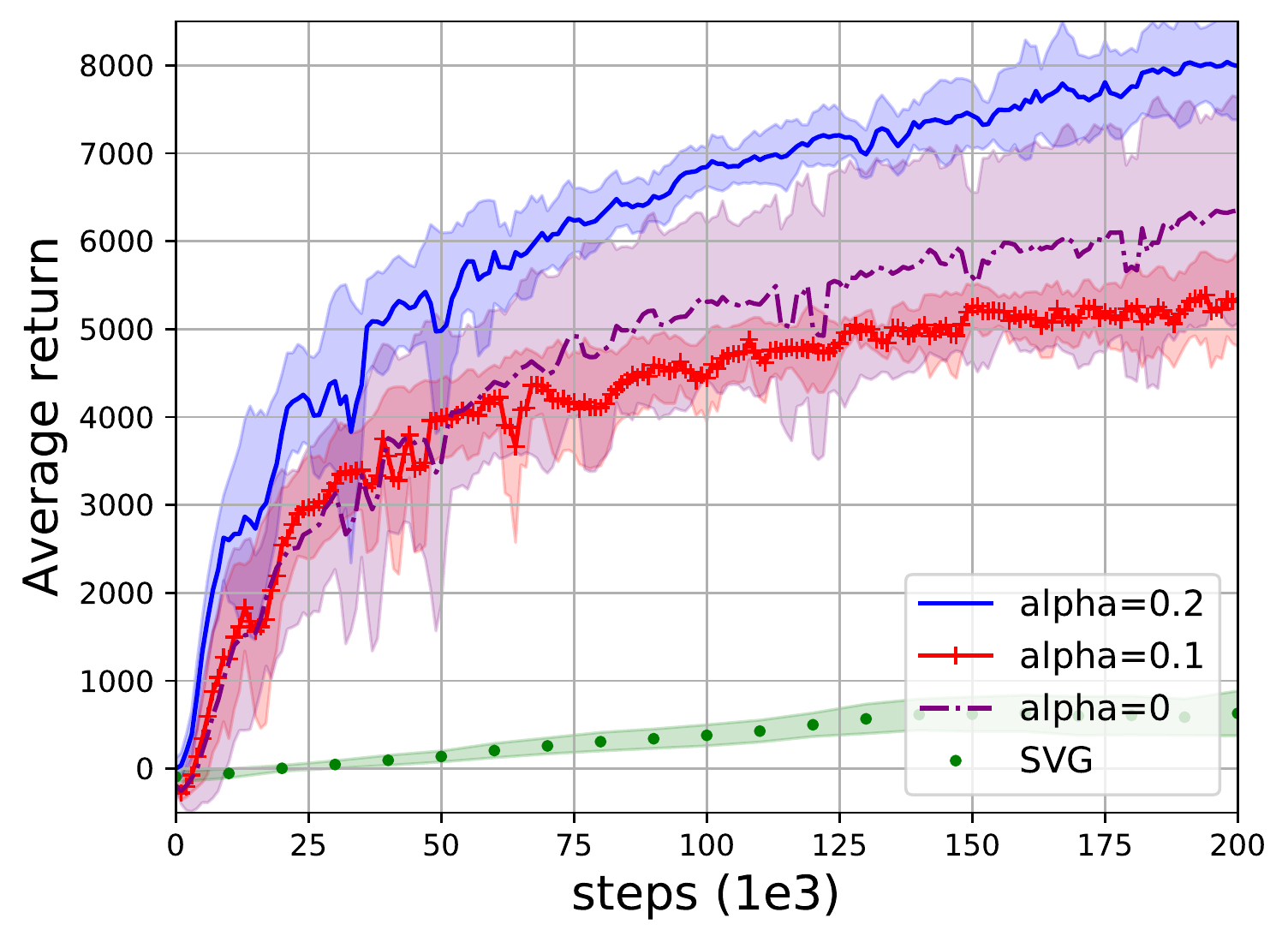}}\quad
\subfloat[Pendulum (value)]{\includegraphics[width=.315\textwidth]{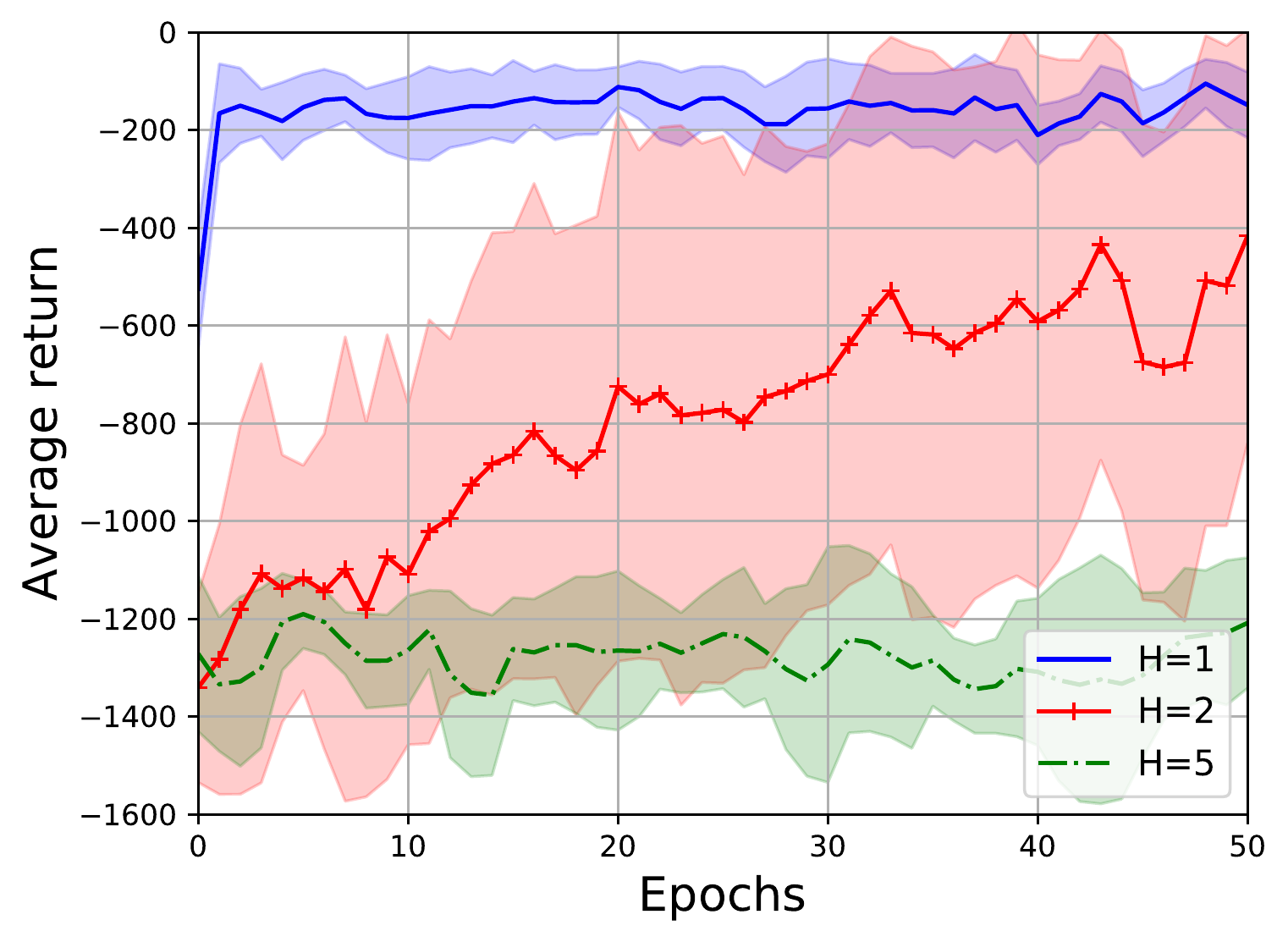}}\quad
\subfloat[HalfCheetah (value)]{\includegraphics[width=.315\textwidth]{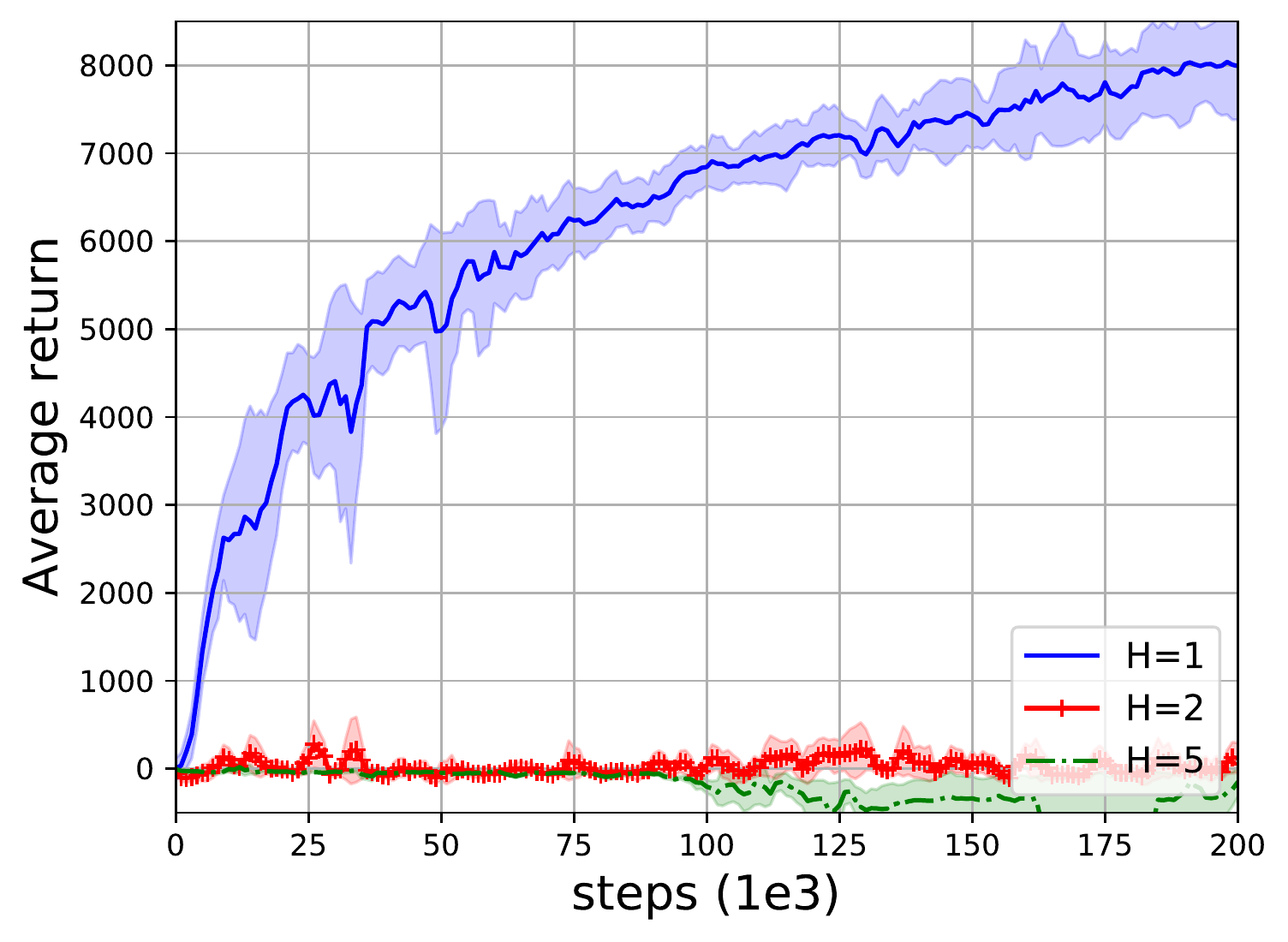}}\quad
 \caption{Ablation study. In (a) we do the ablation study on the effect of the imaginary data on policy learning.  In (b) and (c) we do ablation study on the value function learning with different length of rollout H. The x-axis is the training step and the y-axis is the reward. }
\label{fig:ablation}
\end{figure*}

\subsection{Model Error}

We test the difference between the true model and learned model (using Wasserstein distance), i.e., model error. In particular, we record the learned model by MEMB every 10 epochs and then randomly sample the state action pair $(s,a)$. Feed this pair to the learn model we can obtain the predicted next state $\hat{s}'$, which is used to compared with the true next state $s'$. The error is averaged over state action pair. We do similar things on the reward model. Result are tested on five trials using five different random seeds. They are reported in Figure \ref{fig:model_error}.

 \begin{figure*}[b]
\centering
\subfloat[Transition model error]{\includegraphics[width=.45\textwidth]{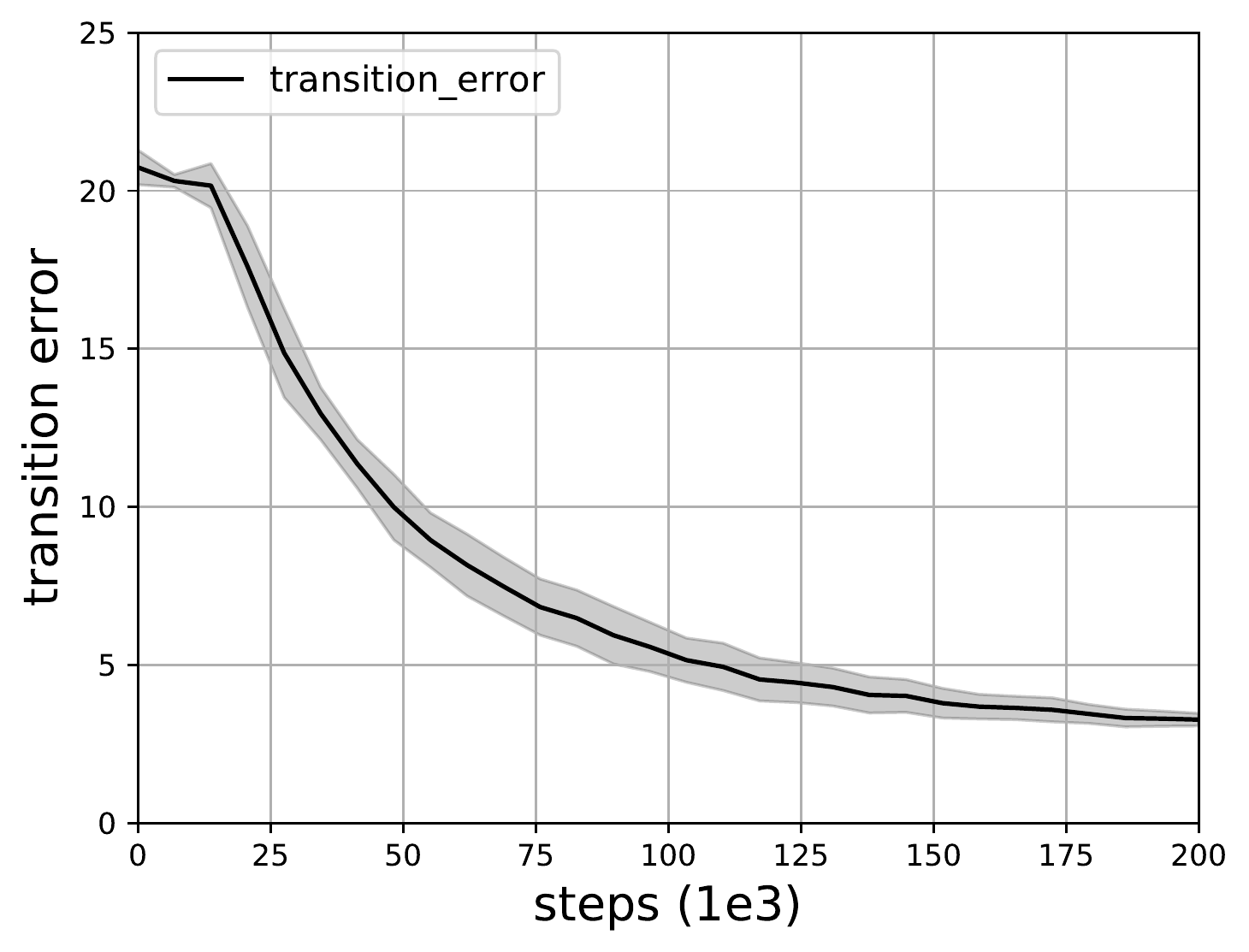}}\quad
\subfloat[Reward model error]{\includegraphics[width=.45\textwidth]{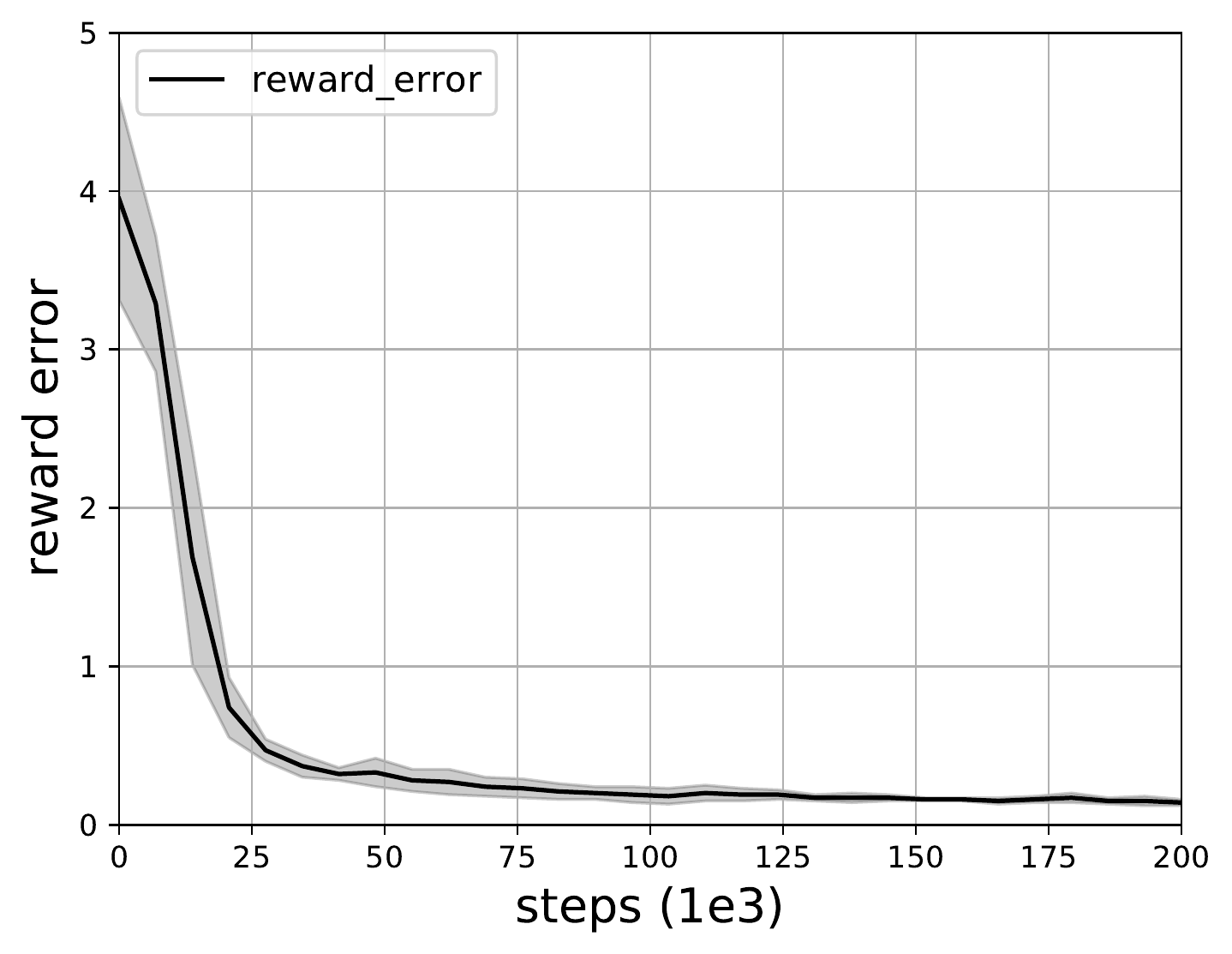}}\quad
 \caption{Model error. We calculate the model error on the environment of HalfCheetah. }
\label{fig:model_error}
\end{figure*}

\subsection{Plug The True Model Into MEMB }

It is interesting to see the performance of MEMB if we plugin the true model in the algorithm.  Since the dynamic in Mujoco is complicated, we just test a simple task pendulum.  Intuitively, MEMB with true model should have better performance than the MEMB with learned one. We verify this in the Figure \ref{fig:true_vs_learn}.

\begin{figure}
    \centering
    \includegraphics[width=.45\textwidth]{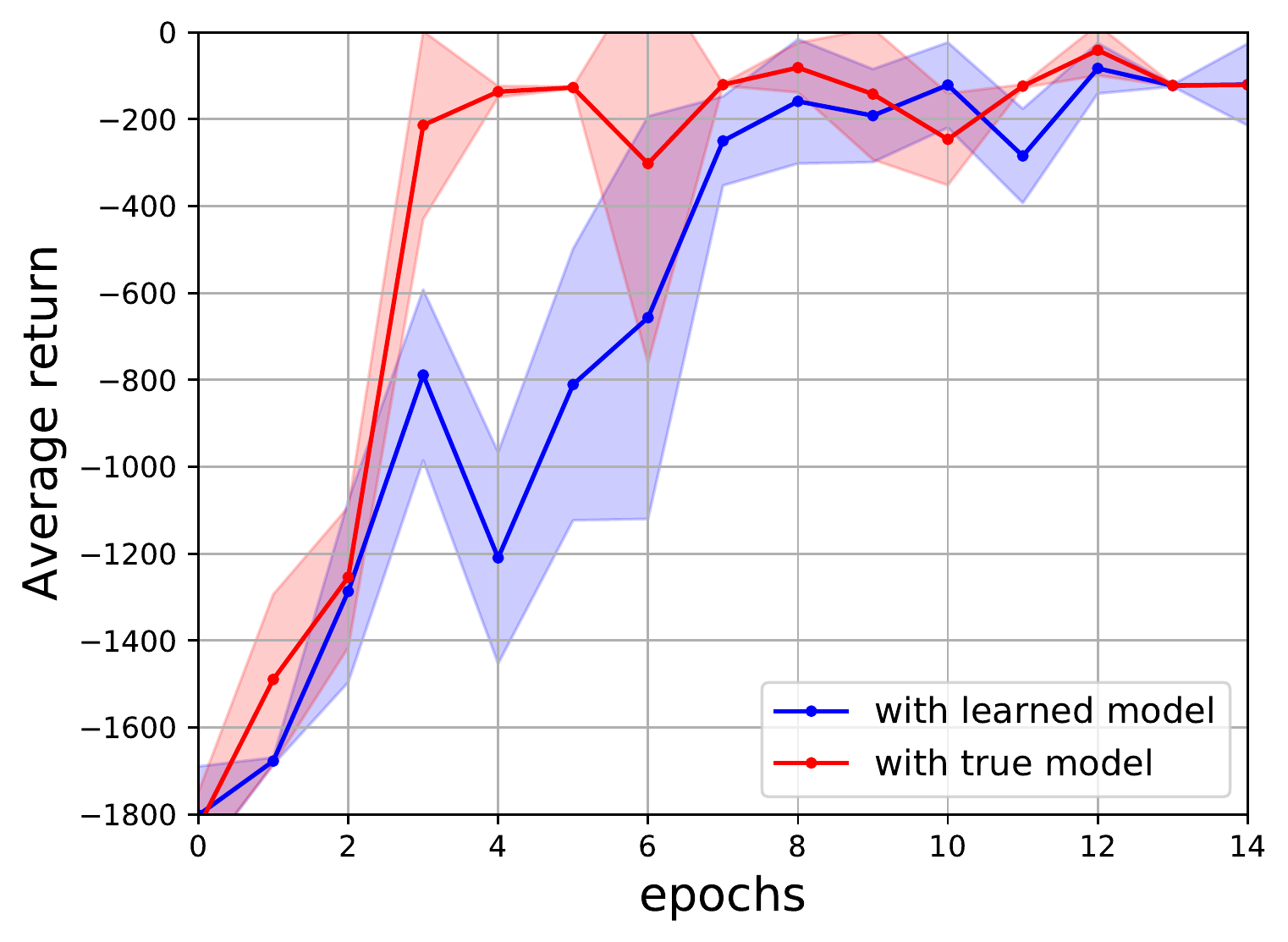}
    \caption{MEMB with the true model vs  MEMB with the learned model in pendulum}
    \label{fig:true_vs_learn}
\end{figure}

\subsection{Asymptotic Performance}

We evaluate the asymptotic behavior of model-free RL (particularly SAC) and MEMB agents through 2000 epoch of training (20M steps) on four simulation environments to see the asymptotic result.
The SAC agent achieves $9453.32\pm 219.34 $ in HalfCheetah, $4182.07\pm63.23$ in Walker2d,  $3433.03\pm18.05$  in Hopper, $310.3 \pm4.60$ in Swimmer.
The MEMB agent achieves $8813.14\pm53.04$ in HalfCheetah,  $3544.3\pm 252.3$ in Walker2d,  $3265\pm12.65$ in Hopper, $307.57\pm 8.97$ in Swimmer. In general, there is a small gap between the asymptotic performance between MBRL and MFRL. It is well expected since the learned model is not accurate.

\section{Environment Overview and Hyperparameter Setting}\label{appendix:experiment}
In this section, we provide an overview of simulation environment in Table \ref{table:environment}. The hyperparameter setting for each environment is shown in Table \ref{table:hyper}. 

\begin{table}[h]
  \begin{center}
    \begin{tabular}{ |c|c|c|c|  }
      \hline
      Environment Name & Observation Space Dimension  & Action Space Dimension & Horizon\\
      \hline
      Pendulum & 3 & 1 & 200\\
      \hline
      InvertedPendulum & 4 & 1 & 1000\\
      \hline
      HalfCheetah & 17 & 6 & 1000 \\
      \hline
      Hopper & 11 & 3 & 1000 \\
      \hline
      Walker2D & 17 & 6 & 1000 \\
      \hline
      Swimmer  &  8& 2 & 1000 \\
      \hline
      Reacher &  11&  2& 50 \\
      \hline
    \end{tabular}
    \caption{The observation space dimension, action space dimension, and horizon for each simulation environment implemented in the experiment and ablation study. }
    \label{table:environment}
  \end{center}
\end{table}

\begin{table}[h]
\footnotesize
  \begin{center}
    \begin{tabular}{ |c|c|c|c|c|c|c|c|  }
      \hline
        & Pendulum & InvertedPendulum & HalfCheetah & Hopper & Walker2D & Swimmer & Reacher\\
      \hline
      Epoch & 50 & 40 & \multicolumn{2}{c|}{200}& 300 & 180  & 50\\
      \hline
      \tabincell{c}{Policy Learning Rate} & \multicolumn{7}{c|}{0.0003} \\
      \hline
      \tabincell{c}{ Value Learning Rate} & \multicolumn{2}{c|}{0.0003} & 0.001 & 0.001 & 0.0003& 0.0003 & 0.0003\\
      \hline
      \tabincell{c}{Model\\ Learning Rate} & \multicolumn{4}{c|}{0.0003} & 0.0001& 0.0001 & 0.0001\\
      \hline
      \tabincell{c}{Alpha value \\ (in entropy term)} & 0.2 & 0.1 & \multicolumn{3}{c|}{0.4} & 0.2 & 0.2\\
      \hline
      \tabincell{c}{environment steps \\ per epoch} & \multicolumn{6}{c|}{1000} & 50\\
      \hline
      \tabincell{c}{Value and Policy\\ Network Architecture} & \multicolumn{7}{c|}{(256,256)}\\
      \hline
      \tabincell{c}{Model \\Network Architecture} & \multicolumn{2}{c|}{(32,16)} & \multicolumn{2}{c|}{(256,128)} & (256,256)& \multicolumn{2}{c|}{(256,128)}\\
      \hline
      \tabincell{c}{Train Actor-critic \\ Times ($m$)} & 5 &1 & \multicolumn{2}{c|}{5} & 3 & 5 &  5 \\
      \hline
    \end{tabular}
    \caption{The hyper-parameter used in training MEMB algorithm for each simulation environment. The number in policy, value, and model network architecture indicate the size of hidden units in each layer of MLP. The ReLu activation function is implemented in all architecture. }
    \label{table:hyper}
  \end{center}
\end{table}

	\section{Proof}
	In this section, we give the proof of the theorem in the main paper. To start with, we give the definition of the Wasserstein distance and its dual form, since we will use it frequently in the following discussion.
	
	Definition: Give a meritc space $(M,d)$ and the set $\mathbb{P}(M)$ of probability measures on $M$, the Wasserstein metric between two probability distributions $\mu_1$ and $\mu_2$ in $\mathbb{P}(M)$ is defined as 
	
	\begin{equation}
	W(\mu_1,\mu_2):=\inf_{j\in \Sigma} \int\int p(x,y)d(x,y)dxdy,
	\end{equation}
	where $\Sigma$ denotes the collection of all joint distributions $p$ with marginal $\mu_1$ and $\mu_2$.
	
	The dual presentation is a special case of the duality theorem of Kantorovich and Rubinstein \cite{villani2010optimal}.
	
	\begin{equation}
	W(\mu_1,\mu_2)=\sup_ {\|f\|\leq 1} \int f(s)(\mu_1(s)- \mu_2(s))ds 
	\end{equation}
	where $\|f\|\leq 1$ means the function $f$ is $1$-Lipschitz.

	The first lemma is well known. It says the Lipschitz constant of a composition function is the product of Lipschitz constants of two functions.
	\begin{lemma}\label{lemma:Lipschitz}
		Define three metric spaces $(M_1,d_1),(M_2,d_2),(M_3,d_3)$. Define Lipschitz function $f:M_2 \rightarrow M_3$ and $g: M_1\rightarrow M_2$ with constant $K_f$, $K_g$. Then $h: f\circ g$ is Lipschitz with constant $K_h\leq K_f K_g$
	\end{lemma}
	\begin{proof}
		\begin{equation}
		\begin{split}
		K_h=&\sup_{x_1,x_2} d_3\big(f(g(x_1),f(g(x_2)))\big)/d(x_1,x_2)\\
		=&\sup_{x_1,x_2}\frac{d_2 (g(x_1),g(x_2))}{d_1(x_1,x_2)}\frac{d_3\big(f(g(x_1),f(g(x_2)) \big)}{d_2(g(x_1),g(x_2))}\\
		\leq &  \sup_{x_1,x_2}\frac{d_2 (g(x_1),g(x_2))}{d_1(x_1,x_2)} \sup_{y_1,y_2}\frac{d_3(y_1,y_2)}{d_2(y_1,y_2)}\leq K_g K_f
		\end{split}
		\end{equation}
	\end{proof}
	\begin{lemma}\label{lemma:distance_joint}
		Suppose we have two joint distribution $p_1(s,a)=p_1(s)\pi_1(a|s)$, $p_2(s,a)=p_2(s)\pi_2(a|s)$. We further assume that $W(p(s_1),p(s_2))\leq \epsilon_m$ and $W(\pi_1(a|s),\pi_2(a|s))\leq \epsilon_\pi$. Then we have 
		$W(p_1(s,a),p_2(s,a))\leq \epsilon_\pi+K_\pi\epsilon_m$.
	\end{lemma}

	\begin{proof}

		Using the triangle inequality, we have
		$$W(p_1(s,a),p_2(s,a))\leq W(p(s_1)\pi_1(a|s),p_1(s)\pi_2(a|s))+W(p_1(s)\pi_2(a|s),p_2(s)\pi_2(a|s)).$$
		
		Now we bound the first term and second term respectively.
		For the first term, according to the dual form of the Wasserstein distance, we have
		\begin{equation}
		\begin{split}
		&W(p_1(s)\pi_1(a|s), p_1(s)\pi_2(a|s) )= \sup_{\|f\|\leq 1} \int\int f(s,a)p_1(s) (\pi_1(a|s)-\pi_2(a|s))dads\\.
		\end{split}
		\end{equation}
		Notice it is easy to verify that if $f(s,a)$ is a 1-Lipschitz function  w.r.t. $(s,a)$, then for a fixed $a$, $f(s,a)$ (we denote it as $f_s(a)$) is also a 1-Lipschitz function w.r.t $a$. Thus $\forall f \in \{f: \|f\|\leq 1 \}$, we have

		\begin{equation}
		\begin{split}
		&\int\int f(s,a) p_1(s)(\pi_1(a|s)-\pi_2(a|s)  )da ds\\
		=& \int p_1(s)\int f_s(a)(\pi_1(a|s)-\pi_2(a|s))da ds\\
		\leq &\int p_1(s) W(\pi_1(a|s),\pi_2(a|s))ds  = \epsilon_{\pi}.
		\end{split}
		\end{equation}
		
		We then bound the second term in the following way.
		\begin{equation}
		\begin{split}
		W(p_1(s)\pi_2(a|s),p_2(s)\pi_2(a|s))=&\sup_{\|f\|\leq 1} \int\int f(s,a)(p_1(s)-p_2(s))\pi_2(a|s)dads\\
		\overset{(1)}{=}&\sup_{\|f\|\leq 1} \int \int f(s,a) \sum g_{\pi}(f_\pi) \mathbbm{1}(f_\pi(s)=a) \big( p_1(s)-p_2(s)\big)dsda\\
		=&\sup_{\|f\|\leq 1} \int \sum_{f_\pi} g_\pi(f_\pi) f(s,f_\pi(s)) (p_1(s)-p_2(s)) ds\\
		\leq& \sum_{f_\pi} g_\pi(f_\pi) \sup_{\|f\|\leq 1} \int f(s,f_\pi(s))(p_1(s)-p_2(s))ds\\
		=&\sum_{f_\pi} g_\pi(f_\pi)K_{\pi}\sup_{\|f\|\leq 1}\int \frac{f(s,f_\pi(s))}{K_\pi } (p_1(s)-p_2(s))ds\\
		\overset{(2)}{=}&\sum_{f_\pi} g_\pi(f_\pi)K_\pi W(p_1(s),p_2(s))\\
		\leq& K_{\pi} \epsilon_m 
		\end{split}
		\end{equation}
		where (1) holds using the assumption $\pi$ is in the Lipschitz class. (2) uses the fact that $\frac{f(s,f_\pi(s))}{K_\pi}$ is 1 Lipschitz, which holds using the similar argument in Lemma \ref{lemma:Lipschitz}.
		
		Combine two pieces together, we obtain the result.
	\end{proof}
	
	\begin{lemma}\label{lemma:distance_state}
		Define $p_{1,\pi_1}(s'|s)=\int p_1(s'|s,a)\pi_1(a|s)da$ and $p_{2,\pi_2}(s'|s)=\int p_2(s'|s,a)\pi_2(a|s) da$. Suppose $W(p_{1}(s'|s,a),p_2(s'|s,a)  )\leq \epsilon_m$, $W(\pi_1(a|s), W(\pi_2(a|s)))\leq \epsilon_\pi$, then we have 
		$$ W(p_{1,\pi_1}(s'|s),p_{2,\pi_2}(s'|s))\leq K_m\epsilon_\pi +\epsilon_m $$.
	\end{lemma}
	
	\begin{proof}
		We define a reference probability distribution $p_{1,\pi_2}(s'|s)=\int p_{1}(s'|s,a)\pi_2(a|s)da$.
		Using the triangle inequaity, we have
		$$W( p_{1,\pi_1}(s'|s)),p_{2,\pi_2}(s'|s))\leq W(p_{1,\pi_1}(s'|s), p_{1,\pi_2}(s'|s))+W(p_{1,\pi_2}(s'|s),p_{2,\pi_2}(s'|s)  ).$$
		
		Thus we just need to bound the two terms on the right hand side.
		
		For the first term, according to the definition of the Wasserstein distance, we have
		\begin{equation}
		\begin{split}
		&W(p_{1,\pi_1}(s'|s),p_{1,\pi_2}(s'|s)  )\\
		\leq& \sup_{\|b\|\leq 1} \int \big(\int p_1(s'|s,a)\pi_1(a|s)-p_1(s'|s,a)\pi_2(a|s)da \big)b(s')ds'da\\
		\overset{(1)}{=}& \sup_{\|b\|\leq 1}\int \int \sum_{f_m} g_m(f_m) (\pi_1(a|s)-\pi_2(a|s)) \mathbbm{1}(f_m(a,s)=s')b(s')ds'da \\
		=& \sup_{\|b\|\leq 1}\sum_{f_m}g_m(f_m)\int (\pi_1(a|s)-\pi_2(a|s)) b(f_m(s,a))da\\
		\leq& \sum_{f_m} g_m(f_m) K_m\sup_{\|b\|\leq 1} \int (\pi_1(a|s)-\pi_2(a|s)) \frac{b(f_m(a,s))}{K_m}da\\
		\overset{(2)}{\leq} & \sum_{f_m}g_m(f_m)K_m W(\pi_1(a|s),\pi_2(a|s))\\
		\leq & K_m \epsilon_\pi.
		\end{split}
		\end{equation}
		where (1) holds using the assumption that the transition model is Lipschitz. (2) holds from the fact that $b(f_m(a,s))/K_m$ is 1-Lipschitz w.r.t. $a$. 
		Then we bound the second term. Again according to the definition of the Wasserstein distance, we have
		\begin{equation}
		\begin{split}
		&W(p_{1,\pi_2}(s'|s),p_{2,\pi_2}(s'|s))\\
		=& \sup_{\|f\|\leq 1} \int \int (p_1(s'|s,a)-p_2(s'|s,a))\pi_2(a|s)f(s')dads'\\
		\leq & \int \pi_2(a|s) \sup_{\|f\|\leq 1}\int (p_1(s'|s,a)-p_2(s'|s,a))f(s')ds'da\\
		\leq & \int \pi_2(a|s) \epsilon_m da\\
		= & \epsilon_m
		\end{split}
		\end{equation}
		
		combine above two pieces together, we obtain the result.
	\end{proof}
	In the next lemma, we would like to bound the Wasserstein distance between distribution $ W(p_{1,\pi_1}^n(s'|s_0),p_{2,\pi_2}^n(s'|s_0))  $, where  $s_0$ is the initial state,\\ 
	$ p_{1,\pi_1}^n(s'|s_0)=\int \int p_{1,\pi_1}^{n-1}(s|s_0)\pi(a|s)p_1(s'|s,a) dads$. 
	\begin{lemma}\label{lemma: distance_n_step}
		Denote $\Delta= \epsilon_m+K_m\epsilon_\pi$ and $\bar{K}=K_m K_\pi$.  Then $W(p^n_{1,\pi_1} (s'|s_0), p^n_{2,\pi_2}(s'|s_0))\leq \Delta \sum_{i=0}^{n-1}\bar{K}^i =\Delta\frac{1-\bar{K}^n}{1-\bar{K}}$
	\end{lemma}
	
	\begin{proof}
		We prove the result by induction. Denote $\delta(n)=W(p^n_{1,\pi_1} (s'|s_0), p^n_{2,\pi_2}(s'|s_0)).$ Thus $\delta(1)=W(p_{1,\pi_1}(s'|s_0),p_{2,\pi_2}(s'|s_0) )$. Using lemma \ref{lemma:distance_state}, we have $\delta(1)\leq \epsilon_m+k_m\epsilon_\pi=\Delta$.
		Using the triangle inequality, we obtain
		\begin{equation}
		\begin{split}
		\delta(n)=&W(p_{1,\pi_1}^n (s'|s_0),p_{2,\pi_2}^n(s'|s_0) )\\
		\leq & W(p_{1,\pi_1}^n(s'|s_0), p_{1,\pi_1}(s'|p_{2,\pi_2}^{n-1} (\cdot|s_0))  ) + W(p_{1,\pi_1}(s'| p_{2,\pi_2}^{n-1}(\cdot|s_0)), P_{2,\pi_2} (s'|p_{2,\pi_2}^{n-1}(\cdot|s_0)))\\
		=& W(p_{1,\pi_1}^n(s'|p^{n-1}_{1,\pi_1}(\cdot|s_0)), p_{1,\pi_1}(s'|p_{2,\pi_2}^{n-1} (\cdot|s_0))  ) + W(p_{1,\pi_1}(s'| p_{2,\pi_2}^{n-1}(\cdot|s_0)), P_{2,\pi_2} (s'|p_{2,\pi_2}^{n-1}(\cdot|s_0))).
		\end{split}
		\end{equation}
		We bound two terms on the right hand side respectively. For the first term,
		we denote $\mu_1:=p^{n-1}_{1,\pi_1}(\cdot|s_0)$ and $\mu_2:=p_{2,\pi_2}^{n-1}(\cdot|s_0)  $ for short. 
		
		Thus, we need to bound $ W(p_{1,\pi_1}^n(s'|\mu_1), p^n_{1,\pi_1}(s'|\mu_2)).$ According to the definition of the Wasserstein distance, we have
		\begin{equation}
		\begin{split}
		&W(p_{1,\pi_1}^n(s'|\mu_1), p^n_{1,\pi_1}(s'|\mu_2))\\
		=& \sup_{\|f\|\leq 1} \int ( p_{1,\pi_1}(s'|\mu_1)-p_{1,\pi_1}(s'|\mu_2) )f(s')ds'\\
		=&\sup_{\|f\|\leq 1} \int \int\int p_{1}(s'|s,a)\pi(a|s)(u_1(s)-u_2(s)) f(s')ds'dads\\
		=& \sup_{\|f\|\leq 1} \int \int \int p_{1}(s'|s,a)\sum_{f_\pi} g_\pi(f_\pi)\mathbbm{1}(f_\pi(s)=a) (u_1(s)-u_2(s)) f(s')ds'dads\\
		=&\sup_{\|f\|\leq 1} \int \int \sum_{f_\pi} g_\pi(f_\pi) p_1(s'|s,f_\pi(s))(\mu_1(s)-\mu_2(s))f(s')ds'ds\\
		=&\sup_{\|f\|\leq 1}\int \int \sum_{f_\pi} g_\pi(f_\pi)\sum_{f_m} g_{m}(f_m)(\mu_1(s)-\mu_2(s)) \mathbbm{1}(f_m(s,f_\pi(s)))f(s')ds'ds\\
		\leq & \sum_{f_\pi} g_{\pi}(f_\pi) \sum_{f_m} g_m (f_m) \sup_{\|f\|\leq 1} \int \int (\mu_1(s)-\mu_2(s)) f(f_m(s,f_\pi(s)))ds\\
		=& \sum_{f_\pi} g_{\pi}(f_\pi) \sum_{f_m} g_m (f_m) K_\pi K_m\sup_{\|f\|\leq 1} \int \int (\mu_1(s)-\mu_2(s)) f(f_m(s,f_\pi(s)))/(K_\pi K_m)ds\\
		=& K_\pi K_m W(\mu_1,\mu_2)\\
		=& \bar{K} \delta(n-1)
		\end{split}
		\end{equation}
		Then we bound the second term $W(p_{1,\pi_1}(\cdot|\mu_2),p_{2,\pi_2}(\cdot| \mu_2)).$ Using the same argument with Lemma \ref{lemma:distance_state}, we have 
		$W(p_{1,\pi_1(\cdot|\mu_2)},p_{2,\pi_2}(\cdot| \mu_2))\leq \epsilon_m+K_m\epsilon_\pi$. We denote $ \Delta:=\epsilon_m+K_m\epsilon_\pi$.
		Combine all pieces together, we have
		\begin{equation}\label{equ:induction_step}
		\delta(n)= \bar{K}\delta(n-1)+\Delta
		\end{equation}
		
		Suppose $\delta(n-1)\leq \delta \sum_{i=0}^{n-2} \bar{K}^i$, we have
		$ \delta(n)\leq \bar{K} \delta(n-1)+\Delta\leq \Delta \sum_{i=0}^{n-1}\bar{K}^i.$
	\end{proof}

	In the next lemma, we bound the long term reward with different model and policy. 
	
	\begin{lemma}\label{lemma:long_term_general}
		let $\eta_1$ be the long term reward induced by the policy $\pi_1(a|s)$ and model $p_1(s'|s,a)$.  $\eta_2$ is the long term reward induced by the policy $\pi_2(a|s)$ and model $p_2(s'|s,a)$. Suppose $r(s,a)$ a $K_r$-Lipschitz function w.r.t $(s,a)$. Then we have $|\eta_1-\eta_2|\leq K_r(\frac{\gamma K_\pi (\epsilon_m+K_m\epsilon_\pi)}{(1-\gamma)(1-\gamma\bar{K})}+\frac{1}{1-\gamma}\epsilon_\pi )
		$
	\end{lemma}
	
	\begin{proof}
		\begin{equation}
		\begin{split}
		\eta_1-\eta_2=&\sum_{n=0}^{\infty} \int \gamma^n (p^n_{1,\pi_1} (s,a)-p^n_{2,\pi_2}(s,a))r(s,a) dsda\\
		=&K_r\sum_{n=0}^{\infty}\gamma^n  \int (p^n_{1,\pi_1} (s,a)-p^n_{2,\pi_2}(s,a))\frac{r(s,a)}{K_r} dsda \\
		\overset{(1)}{\leq}&K_r \sum_{n=0}^{\infty} \gamma^n W(p^n_{1,\pi_1}(s,a),p^n_{2,\pi_2}(s,a))\\
		\overset{(2)}{\leq}& K_r\sum_{n=0}^{\infty}\gamma^n(K_\pi \Delta\frac{1-\bar{K} ^n}{1-\bar{K}} +\epsilon_\pi)\\
		=&K_r\big( \frac{\gamma K_\pi \Delta}{(1-\gamma)(1-\gamma\bar{K})}+\frac{1}{1-\gamma} \epsilon_\pi    \big)\\
		=&K_r(\frac{\gamma K_\pi (\epsilon_m+K_m\epsilon_\pi)}{(1-\gamma)(1-\gamma\bar{K})}+\frac{1}{1-\gamma}\epsilon_\pi ),
		\end{split}
		\end{equation}
		
		where (1) uses the fact that $ r(s,a)/K_r$ is 1-Lipschitz,  (2) uses
		Lemma \ref{lemma:distance_state} and   we have $$W(p^n_{1,\pi_1}(s,a),p^n_{2,\pi_2}(s,a))\leq K_\pi W(p^n_{1,\pi_1}(s), p^n_{2,\pi_2}(s))+\epsilon_\pi.$$
		
		We can derive the same bound for $\eta_2-\eta_1$. Thus the lemma holds.
	\end{proof}
	
	\begin{lemma}\label{lemma:branch_m}
		Assume we run a branched rollout of length m. Before the branch we assume that the Wasserstein distance between $p_1^{pre}(s'|a,s), p_2^{pre} (s'|a,s)$ is bounded by $\epsilon^{pre}_m$ $W(p_1^{pre}(s'|a,s), p_2^{pre} (s'|a,s))\leq \epsilon_m^{pre} $ and similarly after the branch $ W(p_1^{post}(s'|a,s), p_2^{post} (s'|a,s))\leq \epsilon_m^{post}$. The policy difference (w.r.t the Wasserstein distance) is bounded by $\epsilon_\pi^{pre}$ and $\epsilon_\pi^{post}$ respectively . Then the m-step returns are bounded as 
		$|\eta_1-\eta_2|\leq K_r\frac{K_\pi \Delta_{post}(1-\gamma^m)}{(1-\bar{K})(1-\gamma)}-K_r\frac{K_\pi \Delta_{post}(1-(\gamma \bar{K})^m) }{(1-\gamma \bar{K})(1-\bar{K})}+K_r\frac{1-\gamma^m}{1-\gamma}\epsilon_{\pi,post} +\gamma^mK_r[K_\pi\Delta_{m,post}\frac{1}{1-\gamma\bar{K}}+K_\pi \Delta_{pre} \frac{\gamma}{(1-\gamma)(1-\bar{K}\gamma)}+\frac{\epsilon_{\pi,pre}}{1-\gamma}]$
	\end{lemma}
	\begin{proof}
		
		We want to bound  $W(p^t_{1,\pi_1}(s,a), p^t_{2,\pi_2}(s,a) )$. Depending on whether $t$ is larger than $m$ or not, we have two cases. 
		
		For $t\leq m$, using Lemma \ref{lemma: distance_n_step}, we have
		$$W(p_{1,\pi_1}^t(s),p^t_{2,\pi_2}(s)))\leq \Delta_{post} \frac{1-\bar{K}^t}{1-\bar{K}},$$
		where $\Delta_{post}= \epsilon_m^{post}+K_{m}\epsilon_{\pi}^{post}.$
		
		Then we use Lemma 1 and obtain
		
		$$W(p_{1,\pi_1}^t(s,a),p_{2,\pi_2}^t(s,a))\leq K_\pi\Delta_{post} \frac{1-\bar{K}^t}{1-\bar{K}}+\epsilon_\pi^{post} $$

		For $t> m$, we 
		denote $\Delta_{m,post}= W(p_{1,\pi_1}^m(s,a),p^m_{2,\pi_2}(s,a))\leq K_\pi\Delta_{post} \frac{1-\bar{K}^m}{1-\bar{K}}+\epsilon_\pi^{post}. $
		Using \eqref{equ:induction_step} in the proof of lemma \ref{lemma: distance_n_step}, we have
		$$W(p_{1,\pi_1}^t(s,a),p_{2,\pi_2}(s,a))\leq \bar{K}^{t-m} \Delta_{m,post}+\sum_{i=0}^{t-m-1}\bar{K}^i\Delta_{pre}= \bar{K}^{t-m}\Delta_{m,post}+\frac{1-\bar{K}^{t-m}}{1-\bar{K}}\Delta_{pre}.$$
		
		Using Lemma\ref{lemma:distance_joint} again, we have 
		
		$$W(p^t_{1,\pi_1}(s,a),p^t_{2,\pi_2}(s,a))\leq K_\pi\big(\bar{K}^{t-m}\Delta_{m,post}+\frac{1-\bar{K}^{t-m}}{1-\bar{K}}\Delta_{pre})+\epsilon_{\pi}^{pre} $$
		
		Now we are ready to bound the difference between $\eta_1$ and $\eta_2$. We first bound the term  from 0 to m-1.
		\begin{equation}
		\begin{split}
		&\sum_{n=0}^{m-1} \int \gamma^n (p^n_{1,\pi_1} (s,a)-p^n_{2,\pi_2}(s,a))r(s,a) dsda\\
		=&K_r\sum_{n=0}^{m-1}\gamma^n \int (p^n_{1,\pi_1} (s,a)-p^n_{2,\pi_2}(s,a))\frac{r(s,a)}{K_r} dsda \\
		\leq&K_r \sum_{n=0}^{m-1}\gamma^n W(p^n_{1,\pi_1}(s,a),p^n_{2,\pi_2}(s,a))\\
		\overset{(a)}{\leq} &K_r \sum_{n=0}^{m-1} \gamma^n (K_\pi \Delta_{post}\frac{1-\bar{K}^n}{1-\bar{K}}+\epsilon_{\pi,post})\\
		=& K_r\frac{K_\pi \Delta_{post}(1-\gamma^m)}{(1-\bar{K})(1-\gamma)}-K_r\frac{K_\pi \Delta_{post}(1-(\gamma \bar{K})^m) }{(1-\gamma \bar{K})(1-\bar{K})}+K_r\frac{1-\gamma^m}{1-\gamma}\epsilon_{\pi,post} 
		\end{split}
		\end{equation}
		where (a) uses Lemma \ref{lemma: distance_n_step} and Lemma \ref{lemma:distance_joint}.
		
		Then we bound the term from $m$ to infinity.
		
		Denote the error term caused by step $0$ to $m$ as $\Delta_{m,post}$ and
		$\Delta_{m,post}=\Delta_{post} \frac{1-\bar{K}^m}{1-\bar{K}}$. Following similar step as that in $t<m$, we obtain.
		\begin{equation}
		\begin{split}
		&\sum_{n=m}^{\infty} \int \gamma^n (p^n_{1,\pi_1} (s,a)-p^n_{2,\pi_2}(s,a))r(s,a) dsda\\   
		\leq &K_r\sum_{n=m}^{\infty}\gamma^n(K_\pi \bar{K}^{n-m}\Delta_ {m,post}+ K_\pi \frac{1-\bar{K}^{n-m}}{1-\bar{K}}\Delta_{pre} +\epsilon_{\pi,pre})\\
		=&\gamma^m[K_\pi\Delta_{m,post}\frac{1}{1-\gamma\bar{K}}+K_\pi \Delta_{pre} \frac{\gamma}{(1-\gamma)(1-\bar{K}\gamma)}+\frac{\epsilon_{\pi,pre}}{1-\gamma} ]
		\end{split}
		\end{equation}
	\end{proof}

	\begin{proof}[Proof of theorem \ref{theorem:general_result}.]
		
		Let $\pi_D$ be the data collecting policy.  $ \eta(\pi)-\hat{\eta}(\pi)=\eta(\pi)-\eta(\pi_D) + \eta(\pi_D)-\hat{\eta}(\pi).$
		
		Use Lemma \ref{lemma:long_term_general} by setting $\epsilon_m=0$ we have 
		$$|\eta(\pi)-\eta(\pi_D)|\leq \frac{\gamma K_r \bar{K}\epsilon_\pi}{(1-\gamma)(1-\gamma\bar{K})}+\frac{1}{1-\gamma}K_r\epsilon_\pi.$$

		Use Lemma \ref{lemma:long_term_general} again, we have
		
		$$|\eta(\pi_D)-\hat{\eta}(\pi)|\leq \frac{\gamma K_r \bar{K}\epsilon_\pi}{(1-\gamma)(1-\gamma\bar{K})}+\frac{1}{1-\gamma}K_r\epsilon_\pi+\frac{\gamma K_rK_\pi\epsilon_m}{(1-\gamma)(1-\gamma\bar{K})}.$$
		Combine above two results, we have the theorem.
	\end{proof}

	\begin{proof}[Proof of theorem \ref{theorem:branched_result}]
		We define a reference process, where it executes the policy $\pi_D$ under the true dynamics until the branch point and then executes the new policy $\pi$ with true transition model. We denote the return under this scheme as $\eta^{\pi_D,\pi}$. Then we have 
		
		$|\eta(\pi)-\eta^{branch}(\pi)|\leq |\eta(\pi)-\eta^{\pi_D,\pi}|+|\eta^{\pi_D,\pi}-\eta^{branch}(\pi)|$.
		Now we bound two terms on the right hand side respectively.
		
		For the first term $\eta(\pi)-\eta^{\pi_D,\pi} $, the error just come from $\epsilon_\pi^{pre}.$ Plug $\epsilon_\pi^{pre}=\epsilon_\pi$ into lemma \ref{lemma:branch_m} and set all the other error to zeros and we have
		
		$$|\eta(\pi)-\eta^{\pi_D,\pi}|\leq [\frac{\gamma^{k+1} \bar{K}}{(1-\gamma)(1-\bar{K}\gamma)}+\frac{\gamma^k}{1-\gamma}]\epsilon_\pi. $$
		
		For the second term $\eta^{\pi_D,\pi}-\eta^{branch}(\pi)$, the error comes from $\epsilon_m^{post}$. Plug $\epsilon_m^{post}=\epsilon_k$ into Lemma \ref{lemma:branch_m}, we obtain
		$|\eta^{\pi_D,\pi}-\eta^{branch}(\pi)|\leq \frac{K_rK_\pi\epsilon_m(1-\gamma^k)}{(1-\bar{K})(1-\gamma)}-\frac{K_rK_\pi\epsilon_m(1-(\gamma \bar{K})^k)}{(1-\gamma\bar{K})(1-\bar{K})}+\gamma^k \frac{K_rK_\pi\epsilon_m(1-\bar{K}^k)}{ (1-\bar{K})(1-\gamma)}.$
		
		Combine all pieces together, we have the result.

	\end{proof}

	\section{Derivation}\label{appendix:derivation}
	
	We start the derivation with minimization of the KL divergence $KL(\tilde{p}(\tau)||p(\tau))$, where  $ p(\tau)=[p(s_0)\prod_{t=0}^{T}\hat{p}(s_{t+1}|s_t,a_t)]\exp\big(\sum_{t=0}^T \hat{r}(s_t,a_t)\big)$, 
	$\tilde{p}(\tau)=p(s_0) \prod_{t=1}^{T}\hat{p}(s_{t+1}|s_t, a_t)\pi(a_t|s_t).$
	
	\begin{equation}
	\begin{split}
	KL(\tilde{p}(\tau)||p(\tau))=&\mathbb{E}_{\tau\sim \tilde{p}(\tau)} \sum_{t=1}^{T}\big( \hat{r}(s_t,a_t)-\log \pi(a_t|s_t)  \big) \\  
	=&\sum_{\tau} p(s_0) \prod_{t=0}^{T}\hat{p}(s_{t+1}|s_t, a_t)\pi(a_t|s_t)\sum_{t=0}^{T}\big(  \hat{r}(s_t,a_t)-\log \pi(a_t|s_t)  \big)
	\end{split}
	\end{equation}

	Now we optimize KL divergence w.r.t $\pi(\cdot|s_t)$. 
	Considering the constraint $\sum_j \pi(j|s_t)=1 $, we introduce a Lagrangian multiplier $\lambda(\sum_{j=1}^{|\mathcal{A}|}\pi(j|s_t)-1) $ (Rigorously speaking, we need to consider another constraint that each element of $\pi$ is larger than 0, but later we will see the optimal value satisfies this constraint automatically).
	Now we take gradient of $KL(\tilde{p}(\tau)|| p(\tau))+\lambda(\sum_{j=1}^{|\mathcal{A}|} \pi(j|s_t)-1)$ w.r.t $\pi(\cdot|s)$, and obtain 
	
	$$\log \pi^*(a_t|s_t)= \mathbb{E}_{\hat{p}(s_{t+1:T}, a_{t+1:T}|s_t,a_t)}[ 
	\sum_{t'=t}^{T} \hat{r}(s_{t'},a_{t'})-\sum_{t'=t+1}^T \log \pi(a_{t'}|s_{t'}) ] -1+\lambda.$$
	
	Therefore 
	
	$$ \pi^*(a_t|s_t)\propto \exp \big(\mathbb{E}_{\hat{p}(s_{t+1:T}, a_{t+1:T}|s_t,a_t)}[ 
	\sum_{t'=t}^{T}\hat{r}(s_{t'},a_{t'})-\sum_{t'=t+1}^T \log \pi(a_{t'}|s_{t'}) ]  \big).  $$

	Since we know $\sum_{j}\pi(j|s^t)=1$, thus we have 
	
	$$ \pi^*(a_t|s_t)=\frac{1}{Z}\exp \big(\mathbb{E}_{\hat{p}(s_{t+1:T}, a_{t+1:T}|s_t,a_t)}[ 
	\sum_{t'=t}^{T} \hat{r}(s_{t'},a_{t'})-\sum_{t'=t+1}^T \log \pi(a_{t'}|s_{t'}) ]  \big).   $$

	For convenience, we define the soft $V$ function and $Q$ function as that in \citep{levine2018reinforcement}.

	\begin{equation}\label{equ:definition_VQ}
	\begin{split}
	V(s_{t+1}):=\hat{\mathbb{E}} \big[  \sum_{t'=t+1}^T \hat{r}(s_{t'},a_{t'}) -\log \pi(a_t|s_t)|s_{t+1}\big],\\
	Q(s_t,a_t)
	:= \hat{r}(s_t,a_t) + \mathbb{E}_{\hat{p}(s_{t+1}|s_t,a_t})[V(s_{t+1})].  
	\end{split} 
	\end{equation}
	
	Notice it is easy to incorporate the discount factor by defining a absorbing state where each transition have $(1-\gamma)$ probability to go to that state.
	Thus we have
	
	\begin{equation}
	Q(s_t,a_t) = \hat{r}(s_t,a_t) + \gamma \mathbb{E}_{s_{t+1}\sim \hat{p}(s_{t+1}|s_t,a_t)}[V(s_{t+1})],
	\end{equation}  
	\begin{equation}
	V(s_t)=\mathbb{E}_{\pi(a_t|s_t)}[Q(s_t,a_t)-\log\pi(a_t|s_t)]. 
	\end{equation}

\end{document}